\documentclass[12pt, twoside]{article}
\usepackage[margin=1in]{geometry}

\usepackage[utf8]{inputenc} 
\usepackage[T1]{fontenc}    
\usepackage{url}            
\usepackage{booktabs}       
\usepackage{nicefrac}       
\usepackage{microtype}      
\usepackage{subfig}

\usepackage{amsfonts,amsmath,amssymb,amsthm,bm,bbm}

\usepackage{algorithm,algpseudocode}
\usepackage{paralist}

\usepackage{epsfig}
\usepackage{wrapfig}
\usepackage{graphicx}
\usepackage{dsfont,array}

\usepackage{newclude}

\usepackage{xfrac}

\usepackage{appendix}

\usepackage[inline]{enumitem}
\usepackage[colorlinks,
            linkcolor=blue,
            citecolor=blue,
            urlcolor=magenta,
            pdftex
]{hyperref}

\usepackage{multirow}
\usepackage{multicol}
\usepackage{xspace}
\usepackage{epstopdf}
\usepackage{mathtools}
\usepackage{booktabs}
\usepackage{placeins}
\usepackage{authblk}

\newtheorem{definition}{Definition}
\newtheorem{theorem}{Theorem}
\newtheorem{lemma}{Lemma}

\DeclareMathOperator*{\argmax}{arg\,max}
\newcommand\numberthis{\addtocounter{equation}{1}\tag{\theequation}}

\DeclarePairedDelimiterX\Basics[1](){ #1}

\newcommand{\norm}[1]{\left\lVert#1\right\rVert}

\newcounter{const-no}

\def\EE{{\mathbb{E}}}\def\PP{{\mathbb{P}}}

\DeclareMathOperator*{\argmin}{\arg\!\min}

\date{\today}
\title{Contextual Bandits with Latent Confounders: An NMF Approach}
\author{Rajat Sen}
\author{Karthikeyan Shanmugam}
\author{Murat Kocaoglu}
\author{Alexandros G. Dimakis}
\author{Sanjay Shakkottai}
\affil{The University of Texas at Austin}

\begin{document}

\maketitle
\begin{abstract}
Motivated by online recommendation and advertising systems, we consider a causal model for stochastic contextual bandits with a latent low-dimensional confounder. In our model, there are $L$ \textit{observed contexts} and $K$ \textit{arms} of the bandit. The observed context influences the reward obtained through a \textit{latent} confounder variable with cardinality $m$ ($m \ll L,K$). The arm choice and the latent confounder causally determines the reward while the observed context is correlated with the confounder.
Under this model, the $L \times K$ mean reward matrix $\mathbf{U}$ (for each context in $[L]$ and each arm in $[K]$) factorizes into non-negative factors $\mathbf{A}$ ($L \times m$) and  $\mathbf{W}$ ($m \times K$). This insight enables us to propose an $\epsilon$-greedy NMF-Bandit algorithm that designs a sequence of \textit{interventions} (selecting specific arms), that achieves a balance between learning this low-dimensional structure and selecting the best arm to minimize \textit{regret}. Our algorithm achieves a regret of  $\mathcal{O}\left(L\mathrm{poly}(m, \log K) \log T \right)$ at time $T$, as compared to $\mathcal{O}(LK\log T)$ for conventional contextual bandits, assuming a constant gap between the best arm and the rest for each context. These guarantees are obtained under mild sufficiency conditions on the factors that are weaker versions of the well-known Statistical RIP condition.  We further propose a class of generative models that satisfy our sufficient conditions, and derive a lower bound of $\mathcal{O}\left(Km\log T\right)$.  These are the first regret guarantees for online matrix completion with bandit feedback, when the rank is greater than one.  We further compare the performance of our algorithm with the state of the art, on synthetic and real world data-sets.   
\end{abstract}

\section{Introduction}
The study of bandit problems captures the inherent tradeoff between
\textit{exploration} and \textit{exploitation} in online decision
making. In various real world settings, policy designers have the
freedom of observing specific samples and learning a model of the
collected data on the fly; this online learning is instrumental in
making future decisions. For instance in movie recommendations,
algorithms suggest movies to users in order to meet their interests
and simultaneously learn their preferences in an online manner.
Similarly, for product recommendations (e.g. in Amazon) or web
advertisement, there is an inherent tradeoff between collection of
training data for user preferences, and recommending the best items
that maximize profit according to the currently learned model.
Multi-armed bandit problems provide a principled approach to attain
this delicate balance between \textit{exploration} and
\textit{exploitation} \cite{bubeck2012regret}.

The classic $K$-armed bandit problem has been studied extensively for
decades. In the stochastic setting, one is faced with the choice of
pulling one arm during each time-slot among $K$ arms, where the
$k^{th}$ arm has mean reward $U_{k}$. The task is to accumulate a
total reward as close as possible to a \textit{genie} strategy that
has prior knowledge of arm statistics and always selects the optimal
arm in each time-slot. The expected difference between the rewards
collected by the genie strategy and the online strategy is defined as
the \textit{regret}. The expected regret of the state of the art
algorithms~\cite{bubeck2012regret} scales as $O(K\log T)$ when there
is a constant gap between the best arm and the rest.

When side-information is available, a popular model is the contextual
bandit, where the side information is encoded through \textit{observed
  contexts}. In the stochastic setting, at each time an
observed context $s \in [L]$ is revealed, and the observed context
influences the reward statistics of the $K$ arms. Thus, there are
$(K \times L)$ reward parameters $\{U_{sk}\}$ (encoded through the
reward matrix $\mathbf{U}$) that need to be learned, one per each arm
and observed context. Since there are $(K \times L)$ reward
parameters, it has been shown \cite{bubeck2012regret,wu2015algorithms}
that the best expected regret obtainable scales as
$O\left( K L \log T \right).$ 

\noindent \textbf{Netflix Example:} Consider the task of recommending
movies to user profiles on Netflix. A user profile along with the past
browsing history, social and demographic information is the
\textit{observed context}. The list of movies that can be recommended
to any user are the arms of the bandit. In this setting with millions
of users and items, standard contextual bandit algorithms are rendered
impractical due to the $K \times L$ scaling.

Therefore, it is important to exploit that in most practical situations,
the underlying factors affecting the rewards may have a low-dimensional structure. 
Although this low dimensional structure is often not
observable (latent), we will show that it can be leveraged to obtain
better regret bounds. In the context of Netflix, there are millions of
user profiles but the preference of users towards an item may be
represented by a combination of only a handful of \textit{moods},
where these \textit{moods} lie in a much lower dimension. This is
further corroborated by the fact that the Netflix data-set, which has
more than 100 million movie ratings, can be approximated surprisingly
well by matrices of rank as low as 20~\cite{bell2007bellkor}.
Crucially however, these \textit{moods} cannot be directly observed by
a learning algorithm. 

This problem of a contextual bandit with a latent structure has direct
analogy with problem of designing structural \textit{interventions}
(forcing variables to take particular values) in causal graphs, a
class of problems that is of increasing importance in social sciences,
economics, epidemiology and computational
advertising~\cite{pearl2009causality, bottou2013counterfactual}.

\noindent {\bf A Causal Perspective: } A \textit{causal
  model}~\cite{pearl2009causality} is a directed graph that
encodes causal relationships between a set of random variables, where
each variable is represented by a node of the graph (see
Figure~\ref{fig:causalgraphs1}). This example
  has a directed graph with 3 variables, where the variable $Y$ has
  two parents $\{S, A\}.$


  To illustrate the connection between contextual bandits and causal
  models, consider again the Netflix example, which can be mapped to
  the causal graph in Figure~\ref{fig:causalgraphs1}. Here, the {\em
    reward} $Y$ (satisfaction of the user) is causally dependent on
  two quantities -- the \textit{observed context} (user profile in
  Netflix) described by $S$, and the \textit{arm selection} (the
  recommended movie) described by the variable $A.$ Setting $A$ to a
  particular value is equivalent to playing a particular arm (act of
  recommending an item). In this example, $A$ is the \textit{only}
  variable that can be directly controlled by the algorithm; in the
  language of causality this is known as an
  \textit{intervention}~\cite{pearl2009causality} denoted by
  $do(A = a)$.

  More specifically, this contextual bandit setting maps to the causal
  graph problem of affecting a target variable $Y$ (satisfaction of
  users), through \textit{limited interventional capacity} (only being
  able to recommend a movie) when other observable causes (user
  profiles and contextual information) affecting the target variable
  are present but cannot be controlled. This is precisely the model in
  Figure~\ref{fig:causalgraphs1}. An identical structural equation
  model has been defined in Figure 8
  of~\cite{bottou2013counterfactual}.

\begin{figure}[h]
     \centering
     \subfloat[subfigcapskip =10pt][\label{fig:causalgraphs1}Causal
     graph representing regular contextual
     bandits]{\includegraphics[width=0.45\linewidth]{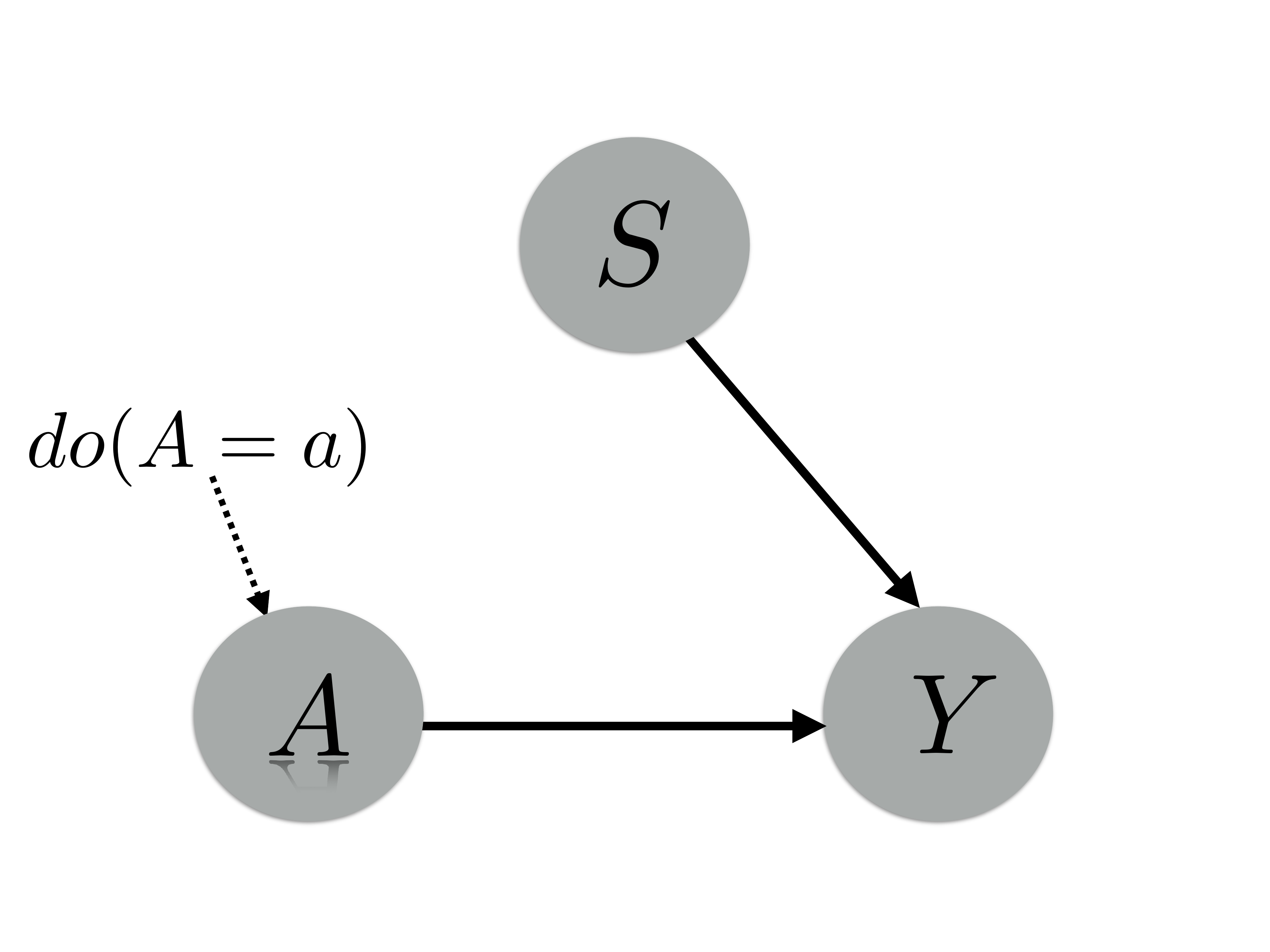}} 
     \qquad
     \subfloat[subfigcapskip =10pt][\label{fig:causalgraphs2}Causal
     graph representing contextual bandits with latent
     confounders.]{\includegraphics[width=0.45\linewidth]{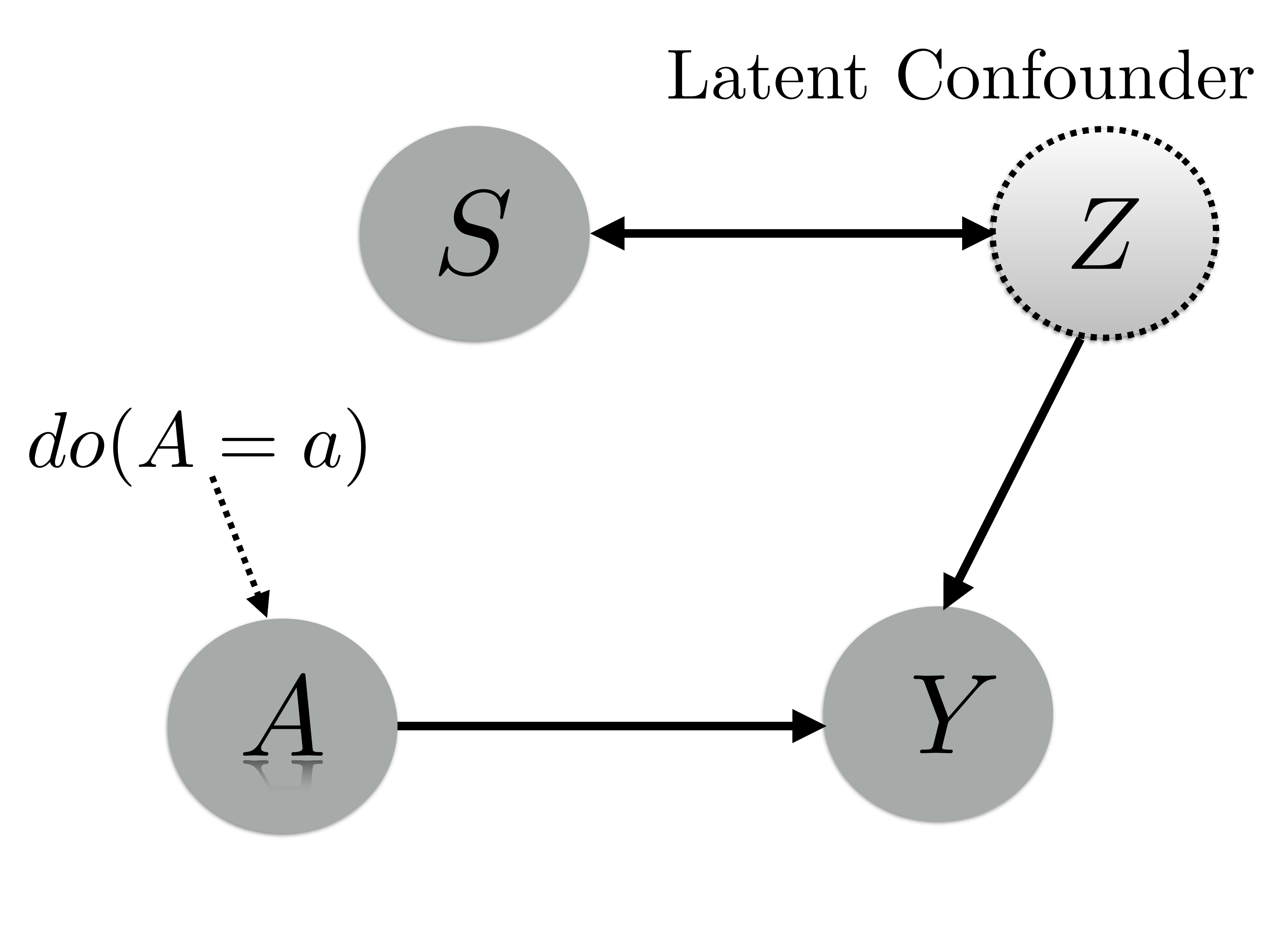}} 
     \caption{Comparison between regular contextual bandits and
       contextual bandits with latent confounders through causal
       graphs.} 
\end{figure}

\noindent {\bf Latent Confounders:} In this causal framework, it is possible to formally capture the
implications of latent effects, such as the \textit{moods} in the
context of Netflix. Consider the modified causal model in
Figure~\ref{fig:causalgraphs2}. The new variable $Z$ denotes a
\textit{latent confounder} (mood) that is causally connected to the
\textit{observed context} and also causally affects the reward $Y$.
The latent confounder $Z$ takes values in $\{1,2,...,m \}$, where
$m \ll L,K$.


The goal here is to develop an efficient algorithm that chooses the
sequence of \textit{limited interventions} (i.e. a sequence of
$do(A=a)$ actions) to achieve a balance between learning this latent
variable (indirectly learning $Z$) from observed rewards, and
maximizing the observed reward under the given (but not intervenable)
observed context $S.$


In the setting of \textit{contextual bandits} with $L$ observed
contexts and $K$ arms, we note that the presence of the
$m$-dimensional latent confounder leads to a factorization of the
$L \times K$ reward matrix $\mathbf{U}$ into non-negative factors
$\mathbf{A}$ (an $L \times m$ matrix) and $\mathbf{W}$ (a $m \times K$
matrix). We leverage this latent low-dimensional structure to develop
an $\epsilon$-greedy NMF-Bandit algorithm that achieves a balance
between learning the hidden low-dimensional structure (indirectly
learning $Z$), and selecting the best arm to minimize regret. In the
setting of \textit{causality}, this result thus demonstrates an
approach to designing a sequence of \textit{interventions} with \textit{limited capacity} to
control a reward variable, in the presence of other (possibly latent)
variables affecting the reward that \textit{cannot} be intervened
upon.

\subsection{Main Contributions}
\label{sec:contrib}
The main contributions of this paper are as follows:

{\bf 1. (Model for Latent Confounding Contexts)} We investigate a
causal model for contextual bandits (Figure~\ref{fig:causalgraphs2}),
which, compared to the conventional model, allows more degrees of
freedom through the unobservable context variable. This allows us to
better capture real-world scenarios. In particular, our model has {\em
  (a)} \textit{Latent Confounders} representing unobserved
low-dimensional variables affecting the mean rewards of the bandit
arms under an observed context; and {\em (b)} \textit{Limited
  Interventional Capacity} signifying that the observed contexts
(eg. user profiles) \textit{cannot} be intervened upon. 

In the contextual bandit setting with $L$ observed contexts and $K$
arms, this translates into a decomposition of the $L \times K$ reward
matrix $\mathbf{U} = \mathbf{AW}$, where $\mathbf{A}$ (non-negative
$L \times m$ matrix) represents the relation between $X$ (observed
contexts) and $Z$ (hidden confounder), while $\mathbf{W}$
(non-negative $m \times K$ matrix) encodes the relation between
$Y$(reward) and $Z$.

{\bf 2. (NMF-Bandit Algorithm)} We propose a latent contextual bandit
algorithm that, in an online fashion, multiplexes two tasks. The {\em
  first task} refines the current estimate of matrix $\mathbf{A}$ by
performing a non-negative matrix factorization (NMF) on the sampled
version of a carefully chosen sub-matrix of the mean-reward matrix
$\mathbf{U}$. The {\em second task} uses the current estimate of
$\mathbf{A}$ and refines the estimate of $\mathbf{W}$ from sampled
versions of several sub-matrices of $\mathbf{U}$.
 
A direct application of results from existing noisy matrix completion
literature is infeasible in the bandit setting.  In the literature, one of the key conditions to derive spectral
norm bounds between the recovered matrix and the ground truth is that
the noise in each entry should be $O(1/K)$ in a $L \times K$
matrix~\cite{hardt2014fast}. In the bandit setting where errors occur
due to sampling, this would lead to a regret of at least would lead to
$O(L K\log T)$ in the presence of sampling errors. We provide further
insights in Section~\ref{sec:thinsight} in the appendix. 

In contrast, our algorithm has much stronger regret guarantees that
scale as 

\noindent $O(L\mathrm{poly}(m,\log K)\log T).$ We show that our
algorithm succeeds when the non-negative matrices $\mathbf{A}$ and
$\mathbf{W}$ satisfy conditions weaker than the well-known statistical
RIP property~\cite{tropp2008norms}. Further, we prove a lower bound
for this setting which is only $\mathrm{poly}(m,\log K)$ factors away
from our upper bound.  This the the first work which has provable
guarantees for matrix completion with bandit feedback for rank greater
than one.

{\bf 3. (Generative Models for $\mathbf{A}$ and $\mathbf{W}$)} We
propose a family of generative models for the factors
$\mathbf{A}$ and $\mathbf{W}$ which satisfy the above sufficient
conditions for recovery. These models are extremely flexible, and
employ a $\textit{random + deterministic}$ composition, where there
can be large number of {\em arbitrary} bounded deterministic entries
(see Section~\ref{sec:models} for details). The remaining random entries in
the matrices are generated from mean-shifted sub-gaussian
distributions (commonly used in the compressive sensing
literature~\cite{foucart2013mathematical}).

Finally, we numerically compare our algorithm with contextual versions
of UCB-1, Thompson Sampling algorithms~\cite{bubeck2012regret} and online matrix factorization algorithms~\cite{katariya2016stochastic} on synthetic and real-world data-sets.

\subsection{Related Work}
\label{sec:rwork}
The current work falls at the intersection of learning of
low-dimensional causal structures and multi-armed bandit problems. We briefly
review the areas of literature that are most relevant to our work.    

\textbf{Contextual Bandit Problems:} There has been significant
progress in contextual bandits both in the adversarial setting and in
the stochastic setting. In the adversarial setting, the best known regret bounds scale as
$O(\sqrt{LKT \log K})$~\cite{bubeck2012regret,wang2005arbitrary} where
$L$ is the number of contexts and $K$ is the number of arms. In the
stochastic regime where there is a constant gap from the best arm, it
can be shown that the regret scales as
$O\left(LK\log T \right)$~\cite{wu2015algorithms}. Contextual bandits with linear payoff functions have been
analyzed in~\cite{agrawal2012thompson, chu2011contextual} in the
adversarial setting, while in~\cite{abbasi2011improved} it has been
analyzed in the stochastic setting. In~\cite{filippi2010parametric}
the authors have expanded this model for the generalized linear model
regime. 

However, these models require one of the low-dimensional features to be known a priori, while our algorithm learns both the features from sampled data. 
Another related line of work is in the online clustering of
bandits~\cite{gentile2014online,maillard2014latent, li2016collaborative}. In this framework, the features of the arms can be directly observed, which is the fundamental difference from our paper.

\textbf{Causality and Bandits:}
Recently, contextual bandit algorithms have found use within the
framework of causality. In \cite{Elias2015bandits}, the authors
investigate a similar latent confounder model. However \cite{Elias2015bandits} does not consider our scaling regime nor provide theoretical guarantees (and has a very different algorithm). 

In \cite{lattimore2016causal}, a causal model for observing feedback has been introduced in the best arm identification regime. However, in their model 
all the variables can be intervened upon. Moreover, the states of all the non-intervened variable including the reward is revealed after the intervention is made. In this work, we focus on a more realistic case where only some of the variables can be intervened upon and in fact some of the variables cannot be observed directly. Further, side information about the observed variables are revealed before an intervention has to be made. The reward is the only extra information that is revealed after each intervention.

\textbf{Online Matrix factorization :} The non-negative matrix factorization (NMF)
problem has generated a lot of interest in the area of semi-supervised
topic modeling. Arora
et al. have shown that if the matrix is separable and has some
robustness properties~\cite{arora2012learning}, then NMF is solvable
efficiently. Since then, there has been a lot of work in proposing
efficient scalable algorithms for NMF, out of which
\cite{gillis2014fast, damle2014random, recht2012factoring} are of
particular interest. There has been some progress in online
NMF~\cite{fevotte2009nonnegative,guan2012online} which aims to update
the features efficiently in a streaming sense. To the best of
our knowledge there has been no work in NMF with bandit feedback with
theoretical guarantees.  \cite{kawale2015efficient, katariya2016stochastic} propose algorithms for online matrix factorization, however they only
have theoretical analysis for the \textit{rank} $1$ case.

%
%
%

\section{Problem Statement and Main Results}
\subsection{System Model}\label{sec:contexts}

{\bf Observed Contexts and Latent Confounders: }\noindent 
We consider a stochastic bandit model represented by the causal graph in Figure~\ref{fig:causalgraphs2}. The variable $S$ denoting the observed context takes values in $\mathcal{S} = \{1,2, \cdots, L \}$, while the variable $A$ determines the arm that has been pulled taking values in $\mathcal{A} = \{1,2, \cdots , K  \}$. The variable $Z$ denotes the latent confounding contexts and takes values in ${\cal Z}=\{z_1,z_2 \cdots z_m\} \subset \mathcal{S}$, where $m \ll L,K$. The causal model results in the bayesian factorization of the joint distribution of $S,Y$ and $Z$.  A natural interpretation is that, at any time nature chooses a latent context $z \in \mathcal{Z}$, and based on that, a context $s \in \mathcal{S}$ is actually observed. We denote the posterior probability of a \textit{latent} context $z$ given an observed context $s$ as, 
\begin{align*}
\mathbb{P} \left(Z = z_i \lvert S = s\right) &= \alpha _{si} , \mbox{  } \forall s \in \mathcal{S} \setminus \mathcal{Z} , z_i \in \mathcal{Z}, \\
\alpha_{sj} &= \mathds{1}\left\{ s = z_j\right\}  \mbox{  } \forall s,z_j \in \mathcal{Z}
\end{align*}
Let $\mathbf{A}$ be the matrix with elements $\alpha_{si}$ where $s \in \{1 \cdots L\}$ and $i \in \{1,2 \cdots m\}$. Please note that the sub-matrix corresponding to the row indices in ${\cal Z}$ from an identity matrix $\mathbf{I}_{m \times m}$. This is essentially the well-known \textit{separability} condition~\cite{recht2012factoring}. We also define the marginal probability of observing a context $s \in S$ as
$\mathbb{P} \left( S = s\right) = \beta _{s},  \mbox{  } \forall s \in \mathcal{S}.$
This specifies the joint distribution of the \textit{latent} context $Z$ and the observed context $S$.

{\bf Bandit Setting: }\noindent
In this setting the contextual bandit problem can be described as follows: {\em (i)}
At each time $t$ the algorithm observes a context $S_t = s_t\in \mathcal{S};$ {\em (ii)}
After observing the context the algorithm selects an arm $A_t = a_t \in \mathcal{A}$ which is the \textit{intervention} $do(A = a_t)$; and {\em (iii)} 
The algorithm then obtains a Bernoulli reward $Y_{t}$ with mean $U_{s_t,a_t}$.  
The mean rewards $U_{s_t,a_t}$ have a latent structure described in the next subsection. 

{\bf Rewards: }\noindent
When an observed context $s$ is provided, the reward for arm $k$ depends only on the latent variables. Consider an $m \times K$ reward matrix $\mathbf{W}$. $W_{ik}$ specifies the mean reward for arm $k$ when the latent context is $z_i$. For all observed contexts $s \in {\cal S}$, the mean rewards are given by the matrix $\mathbf{U}$. This is given by:
  \begin{equation*}
      U_{sk} = \sum_{i} \mathbb{P} \left(Z = z_i \lvert S = s \right) W_{ik} = \sum_i \alpha_{si}W_{ik}.
  \end{equation*}
Therefore, we have $\mathbf{U} = \mathbf{A}\mathbf{W}$. Since the latent contexts ${\cal Z}$ are also a subset of observed contexts, the matrix $\mathbf{A}$ contain a $\mathbf{I}_{m \times m}$ sub-matrix. This is equivalent to the separability condition and is widely used in the NMF literature (see \cite{gillis2014fast}). $\mathbf{A}$ represent the relation $S \longleftrightarrow Z$ while the matrix $\mathbf{W}$ denotes the relation $Z \longrightarrow Y$ in the causal model of Figure~\ref{fig:causalgraphs2}. 

{\bf Regret: }\noindent
The goal is to minimize regret (also known as pseudo-regret \cite{bubeck2012regret}) when compared to a \textit{genie} strategy which knows the matrix $\mathbf{U}$. Let us denote the best arm under a context $s \in \mathcal{S} $ by $k^*(s)$ and the corresponding reward by $u^*(s)$. Now, we are at a position to define the regret of an algorithm at time $T$, 
\begin{equation}
R(T) = \sum _{s \in \mathcal{S}} \sum _{\left\{ t \in [T] : S_t = s \right\}} \left( u^*(s) - \mathbb{E}\left[ Y_t\right] \right)
\end{equation}
Note that the \textit{genie} policy always selects the arm $k^*(s)$ when $S_t = s$. The class of policies we optimize over are agnostic to the true reward matrix $\mathbf{U}$ and ${\cal Z},$ however we assume that $m$ (the latent dimension) is a known scalar parameter.  We work in the problem dependent setting, where there is a gap (bounded away from zero) between the mean reward of the best arm and the second best for every observed context. Let the gap $(\Delta)$, be defined as,
\begin{align*}
\Delta = \min _{s \in [L]} \min_{k \neq k^*(s)} u^*(s) - U_{sk}. 
\end{align*}

\subsection{Notation}

We denote matrices by bold capital letters (e.g. $\mathbf{U}$ ) and vectors with bold small letters (e.g. $\mathbf{x}$ ). For an $L \times K$ matrix $\mathbf{U}_{S,:}$ denotes the sub-matrix restricted to the rows in $S \subset [L]$, while  $\mathbf{U}_{:,R}$ denotes the sub-matrix restricted to the columns in $R \subset [K]$.  $\sigma_{m}(\mathbf{P}$) denotes the $m$-th smallest singular value of $\mathbf{P}$. $\norm{\mathbf{x}}_{p}$ denotes the $\ell_p$-norm of $\mathbf{x}$. For, a matrix  $\norm{\mathbf{U}}_{\infty,1}$ refers to the maximum $\ell_1$-norm among all the rows while $\norm{\mathbf{U}}_{2}$ and $\norm{\mathbf{U}}_{F}$ denotes its spectral  and Frobenius norms respectively. $\norm{\mathbf{U}}_{\infty,\infty}$ denotes the maximum absolute value of an element in the matrix. $\mathrm{Ber}(p)$ denotes a Bernoulli random variable with mean $p$.   
   
\subsection{Main results} 
\label{sec:mainresults}  
We first provide few definitions before presenting our main results.
\begin{definition} \label{def:nullsing}
Consider an $m \times m'$ matrix $\mathbf{P}$ with $m' \geq m$. Define $\psi_{m}(\mathbf{P})=\inf \limits_{\mathbf{a}\neq 0:\mathbf{a}^T\mathbf{1}=0} \frac{\lVert \mathbf{a}^T \mathbf{P} \rVert_2}{\lVert \mathbf{a}  \rVert_2} $.
\end{definition}

\begin{definition} \label{def:nullsing2}
Consider an $m \times m'$ matrix $\mathbf{P}$ with $m' \geq m$. Define $\psi^1_{m}(\mathbf{P})=\inf \limits_{\mathbf{a}\neq 0:\mathbf{a}^T\mathbf{1}=0} \frac{\lVert \mathbf{a}^T \mathbf{P} \rVert_1}{\lVert \mathbf{a}  \rVert_1} $.
\end{definition}

In our work, we require the matrices ($\mathbf{W}$ and $\mathbf{A}$) to satisfy some weaker versions of the `statistical RIP property' (RIP - restricted isometry property). This property has been well studied in the sparse recovery literature \cite{barg2015restricted, chretien2016small,tropp2008conditioning,tropp2008norms,chretien2012invertibility}. Statistical RIP property is a randomized variant of the well-known RIP condition ~\cite{foucart2013mathematical}. RIP requires the extreme singular values to be bounded for sub-dictionaries formed by \textit{any} $k$ columns (or rows) of a dictionary for a suitable $k$. Statistical RIP property is a weaker probabilistic version where extreme singular values need to be bounded for random sub-dictionaries with high probability when $k$ random columns are chosen out of a dictionary to form the random subdictionary. We note that this same property goes by different names such as weak RIP property \cite{chretien2016small} and quasi-isometry property \cite{chretien2012invertibility} in the literature. The terminology we adopt in this work is from \cite{barg2015restricted}.

\begin{definition} {\bf (Statistical RIP Property - StRIP)}
An $L \times m$ matrix ($L \geq m$) $\mathrm{P}$, whose rows have unit $\ell_2$ norm, satisfies the $\ell_2$-\textit{Statistical RIP Property} ($\ell_2$-StRIP) with constants $(\epsilon, \rho, m')$, if 
\begin{align*}
& \mathrm{Pr}_{|S| = m'} ( 1- \rho \leq \sigma_{\mathrm{min}} \left( \mathbf{P}_{S,:} \right) \leq \sigma_{\mathrm{max}} \left( \mathbf{P}_{S,:} \right) \leq 1+\rho) \geq 1- \epsilon,
 \end{align*}
where the probability is taken over sampling a set $S$ of size $m'$ uniformly from $[L]$.
\end{definition}

In our work, we only need a weaker version of StRIP condition to hold. We only need that the smallest singular value be bounded below for random sub-matrices and we work with un-normalized matrices. Hence, we have the following version which we will use: 

\begin{definition} {\bf ($\ell_2$ Weak Statistical RIP Property - $\ell_2$-WStRIP)}
An $L \times m$ matrix ($L \geq m$) $\mathrm{P}$  satisfies the $\ell_2$-\textit{Weak Statistical RIP Property} ($\ell_2$-WStRIP) with constants $(\epsilon, \rho, m')$ if $ \mathrm{Pr}_{|S| = m'} \left(  \sigma_{\mathrm{min}} \left( \mathbf{P}_{S,:} \right)  \geq \rho \right) \geq 1- \epsilon$ where the probability is taken over sampling a set $S$ of size $m'$ uniformly from $[L]$.
\end{definition}

For one of the matrices among $\mathbf{W}$ and $\mathbf{A}$, we need its random sub-matrices to satisfy weaker RIP-like conditions in the $\ell_1$ sense. 
\begin{definition} {\bf ($\ell_1$ Weak Statistical RIP Property - $\ell_1$-WStRIP)}
An $m \times K$ matrix ($K \geq m$) $\mathrm{P}$ satisfies the $\ell_1$-weak statistical RIP property ($\ell_1$-WStRIP) with constants $(\epsilon, \rho, m')$ if $ \mathrm{Pr}_{|S| = m'} (\psi^1_{m}(\mathbf{P}_{:,S} ) \geq \rho) \geq 1- \epsilon$ where the probability is taken over sampling a set $S$ of size $m'$ uniformly from $[K]$.
\end{definition}

In what follows, we assume that $\mathbf{W}$ satisfies $\ell_1$-WStRIP and $\mathbf{A}$ satisfies $\ell_2$-WStRIP. Note that in Section~\ref{sec:models} we provide reasonable generative models for $\mathbf{W}$ and $\mathbf{A}$ that satisfy these conditions with high probability. 
 
Now we are at a position to state Theorem~\ref{thm:existence} which shows the existence of an algorithm for the latent contextual bandit setting, with regret that scales at a much slower rate than the usual $O\left(LK\log T \right)$ guarantees. 
\begin{theorem}
\label{thm:existence}
Consider the bandit model with reward matrix $\mathbf{U} = \mathbf{AW}$. Suppose $\mathbf{A}$ is separable~\cite{recht2012factoring}. Let $\mathbf{A}$ satisfy $\ell_2$-WStRIP with constants $(\delta/L, \rho_2, m_1')$ while $\mathbf{W}$ satisfies $\ell_1$-WStRIP with constants $(\delta, \rho_1, m_2')$. Let $m' = \max(m'_1,m'_2) = \Theta(m\log(K))$. Suppose $\beta _s = \Omega(1/L)$ for all $s \in [L]$. We also assume that $L = \Omega(K\log(K))$. Then there exists a randomized algorithm whose regret at time $T$ is bounded as,
\begin{equation}
R(T) = O \left( L\frac{\mathrm{poly}(m,\log(K))}{\Delta ^2}\log(T)\right)
\end{equation} 
with probability at least $1 - \delta$. Here, $\mathrm{poly}(m,\log(K)) = O \left(m^5 \log ^2 K \right)$. 
\end{theorem}
We present an algorithm that achieves this performance in Section~\ref{sec:algo}. This theorem is re-stated as  Theorem~\ref{thm:mainub} in the appendix which has greater details specific to our algorithm. It should be noted that in practice our algorithm has much lesser regret than $O(Lm^5 \log T)$. This can be observed in Section~\ref{sec:simulations}, where our algorithm performs well even if we set the  $\textit{explore}$ rate much lower than what is prescribed.

\textbf{Remark:} In prior works \cite{barg2015restricted, chretien2016small,tropp2008conditioning,tropp2008norms,chretien2012invertibility} the statistical RIP property was established by relating it to the incoherence parameter $\mu$ of a matrix $\mathbf{B}$ which is defined as $\mu \left( \mathbf{B} \right) = \max \limits_{i \neq j} \lvert \mathbf{b}_i^T \mathbf{b}_j \rvert $. In some works, the average of these incoherence parameters has been used instead. We note that matrices $\mathbf{A}$ and $\mathbf{W}$ are non-negative. Hence, directly using analysis based on controlling dot-products among rows and columns is not useful in this scenario. Hence, we propose generative models for $\mathbf{A}$ and $\mathbf{W}$ that satisfy the properties listed above with high probability even when they are not incoherent. We also explain why these generative models are extremely reasonable for our setting.  

\subsection{Generative Models for $\mathbf{W}$  and $\mathbf {A}$ } \label{sec:models}
 We briefly describe our semi-random generative models for $\mathbf{W}$ and $\mathbf{A}$ that satisfy the weak statistical RIP conditions. We refer to Section~\ref{sec:detmodels} for a more detailed discussion of the generative models. 
\begin{enumerate}
\item \textit{Random+Deterministic Composition}:  A significant fraction of entries of $\mathbf{W}$ and $\mathbf{A}$ are {\em arbitrarily
  deterministic}. $O \left( 1/m \right)$ fraction of columns of $\mathbf{W}$ and $O \left( 1\right)$ fraction of rows of $\mathbf{A}$ are deterministic. In addition, we assume that a sub-matrix in the deterministic part of $\mathbf{A}$ is an identity matrix to account for the separability condition ~\cite{recht2012factoring}. The rest of the entries are mean shifted, bounded
sub-gaussian random variables with some additional mild conditions. 
Uniform prior on reward that has been used in
bandit setting~\cite{kaufmann2012bayesian} reduces to a special case
of this model.

\item \textit{ Bounded randomness in the random part:} The random
  entries of both $\mathbf{W}$ and $\mathbf{A}$ are in ``general
  position'', i.e., they arise from mean shifted bounded sub-gaussian
  distributions (see Section~\ref{sec:detmodels}, and also
  \cite{foucart2013mathematical} for similar conditions in compressed
  sensing literature). The mean shifts in the random parts of $\mathbf{A}$ and $\mathbf{W}$, the support of the sub-gaussian randomness satisfy some technical conditions to 
  make sure that row sum of $\mathbf{A}$ is $1$ and to ensure that the weak statistical RIP conditions are satisfied. 
 \end{enumerate}
One of our main results is stated as Theorem~\ref{thm:goodW}, which
implies that if $\mathbf{W}$ comes from our generative model then with
high probability projecting it onto a small random subset of its
columns preserves the $\alpha$-robust simplical
property~\cite{recht2012factoring} which is a key step in our
algorithm.  
\begin{theorem}
\label{thm:goodW}
   Let $m' \geq \frac{512 }{21\tilde{c}} m \log (eK)$.  Let $\mathbf{W}$ follow the random model in Section \ref{sec:detmodels}. $\mathbf{W}$ satisfies ($\ell_1$-WStRIP) with constants $(2\exp(-c_1 \log (eK)), \left(\frac{13}{60} \right) \frac{\sqrt{15m'}}{\sqrt{8m}}, 2m')$ with probability at least $1- \exp(-c'_1 \log (eK))$. Here, $c_1,c'_1$ are constants that depend on the sub-gaussian parameter $c(q)$ that depends on the variance in the model for $\mathbf{W}$.
\end{theorem}
In Theorem~\ref{thm:goodA}, we follow very similar techniques to prove that small random subsets of rows of $\mathbf{A}$ have singular values bounded away from zero with high probability if $\mathbf{A}$ is drawn from our generative model. 
\begin{theorem}
\label{thm:goodA}
   Let $m' \geq \frac{512 }{21\tilde{c}} m \log (eL)$.  Let $\mathbf{A}$ follow the random model in Section \ref{sec:detmodels}. $\mathbf{A}$ satisfies ($\ell_2$-WStRIP) with constants $(2\exp(-c_2 m \log (eL)), \frac{1}{20}\frac{\sqrt{m'}}{m} , 2m')$ with probability at least $1- \exp(-c'_2 m\log (eL))$. Here, $c_2',c_2$ are constant the depends on the sub-gaussian parameter $c(q)$ that depends on the variance in the model for $\mathbf{A}$.
\end{theorem}
The proof of these theorems are available in the appendix in Section~\ref{sec:projall}.

\subsection{Lower Bound}
\label{sec:lbb}
We prove a problem-specific regret lower bound for a specific class of
parameters $(\mathbf{U},\mathbf{W},\mathbf{A})$ which is only a
$\mathrm{poly}(m,\log(K))/\Delta$ factor away from the upper bound achieved
by our algorithm. The lower bound holds for all policies in the class
of $\alpha$-consistent policies \cite{salomon2011regret} defined below.  
\begin{definition}
\label{defn:consistent}
A scheduling policy is said to be $\alpha$-consistent if given any problem instance $\mathbf{U}$ we have,
$
\mathbb{E} \left[\sum _{\left\{ t \in [T] : S_t = s \right\}} \mathds{1} \left\{ X_t = k \right\} \right] = O(T(s)^{\alpha})
$
for all $k \neq k^*(s)$ and $s \in \mathcal{S}$, where $\alpha \in (0,1)$ and $T(s) = \sum_{t=1}^{T} \mathds{1} \left\{S_t = s \right\}$ 
\end{definition}


\begin{theorem}
There exists a problem instance $(\mathbf{U},\mathbf{A},\mathbf{W})$ with $\beta_s = \Omega(1/L)$ for all $s \in \mathcal{S}$ such that the regret of any $\alpha$-consistent policy is lower-bounded as follows,
\begin{align*}
R(T) &\geq (K-1)mD(\mathbf{U})((1-\alpha)(\log(T/2m)  - \log(L/m)) - \log(4KC))
\end{align*}
for any $T > \tau$, where $C,\tau$ are universal constants independent of problem parameters and $D(\mathbf{U}) = O\left( 1/\Delta \right)$ is a constant that depends on the entries of $\mathbf{U}$ and is independent of $L,K$ and $m$. 
\end{theorem}

The proof of this theorem has been deferred to the appendix in Section~\ref{sec:lboundproof} where we specify the class of problem parameters for which we construct this bound.


\section{NMF-Bandit Algorithm}
\label{sec:algo}
In this section we present an $\epsilon$-greedy algorithm that we call
NMF-Bandit algorithm. Our algorithm takes advantage of the
the low-dimensional structure of the reward matrix.
The algorithm
\textit{explores} with probability $\epsilon_t;$ in this case it
samples from specific sets of arms (to be specified later). Otherwise
w.p. $(1- \epsilon_t)$ it \textit{exploits,} i.e., chooses the best
arm based on current estimates of rewards to minimize regret. A detailed pseudo-code of our algorithm has been presented as
Algorithm~\ref{alg:LCB} in the appendix. The key
steps in the algorithm are as follows.


{\rule{\linewidth}{0.2pt}}

\noindent \textbf{(a)} At each time $t$ and with probability
$\epsilon _t,$ the algorithm \textit{explores,} i.e. it randomly
performs one of these two steps:

\noindent {\bf Step 1 -- (Sampling for NMF in low dimensions to
  estimate $\mathbf{A}$):} Given that it \textit{explores}, with probability
  $\alpha$ it samples a random arm from a subset $S \subset [K]$ of
  arms. $\lvert S \rvert = 2m'$ for $m' = O(m\log(K))$. The set $S$ is
  a randomly chosen at the onset and kept fixed there after.  This is Step 6 of Algorithm~\ref{alg:LCB}.


  \noindent {\bf Step 2 -- (Sampling for estimating $\mathbf{W}$):}
  Otherwise with probability $(1- \alpha),$ it samples in a context
  dependent manner.  If the context at the time is $s_t$, the algorithm samples one arm
  at random from a set of $m$ arms given by $R(s_t)$ (the selection of
  these sets are outlined below). The context specific sets of arms
  are designed at the start of the algorithm and held fixed there
  after.  This is Step 7 of Algorithm~\ref{alg:LCB}.


\noindent \textbf{(b)} Otherwise with probability $(1 - \epsilon _t)$
it \textit{exploits} by performing Step 3 below.

\noindent {\bf Step 3 -- (Choose best arm for current observed
  context):} Compute estimate $\hat{\mathbf{A}}(t)$ as detailed in Step 10 of Algorithm~\ref{alg:LCB}, using $\mathrm{Hottopix}$. Estimate $\hat{\mathbf{W}}(t)$ as detailed in Step 11 of Algorithm~\ref{alg:LCB}. Let $\hat{\mathbf{U}}(t) = \hat{\mathbf{A}}(t)\hat{\mathbf{W}}(t)$. The
algorithm plays the arm given by
$\argmax_{k \in [K]} \hat{\mathbf{U}}(t)_{s_t,k},$ i.e., the best arm
for the observed context according to current estimates.

{\rule{\linewidth}{0.2pt}}

For solving the NMF to obtain $\hat{\mathbf{A}}(t),$ we use a
robust version of Hottopix~\cite{recht2012factoring, gillis2014fast}
as a sub-routine.  
 Now, we briefly touch upon the construction of the context specific sets of arms in Step 2 of the
\textit{explore} phase. These sets have been defined in detail
in Section~\ref{sec:algdetails} in the appendix. Let $l = \lfloor K/m
\rfloor$. A set $R \subset [L]$ of contexts is sampled at random, such that
$\lvert R \rvert = 2(l+1)m'$ at the onset of the algorithm. We
partition $R$ into $l+1$ contiguous subsets
$\left\{S(1),S(2),...,S(l+1) \right\}$ of size $2m'$ each. In Step 2
of \textit{explore}, if $s_t \in S(i)$, then $R(s_t) =
\{(i-1)m,(i-1)m+1,\cdots \max (im-1,K) \}$. If $s_t \notin S(i)$ for
all $i \in [l+1]$, then the algorithm is allowed to pull any arm at
random, and these samples are ignored.  

A more detailed version of our main result (Theorem~\ref{thm:existence}) has been
provided in  (Theorem~\ref{thm:mainub}) in the appendix, along with a
detailed proof. Theorem~\ref{thm:mainub} exactly specifies the
algorithm parameters $\epsilon _t$, $\alpha$ and $m'$ under which we
obtain the regret guarantees. We provide some key theoretical insights and a brief proof sketch in Section~\ref{sec:thinsight} in the appendix. In particular we discuss in detail why usual matrix completion techniques would 
fail to provide regret guarantees that are $o(KL\log(T))$. We explain the challenges of dealing with sampling noise and how we overcome them through careful design of the arms to \textit{explore}.


%
%
%
%
\section{Simulations}
\label{sec:simulations}
We validate the performance of our algorithm against various benchmarks on real and synthetic datasets. We compare our algorithms against contextual versions of UCB-1~\cite{bubeck2012regret} and Thompson sampling~\cite{agrawal2012analysis}. To be more precise, these algorithms proceed by treating each context separately and applying the usual $K$-armed version of the algorithms to each context. We also compare the performance of our algorithm to this recent algorithm~\cite{katariya2016stochastic} for stochastic rank $1$ bandits. In \cite{katariya2016stochastic} the problem setting is different. Therefore, whenever we compare the performance with this algorithm the experiments have been performed in the setting of \cite{katariya2016stochastic}, which we call \textbf{S2}. The more realistic setting of our paper will be denoted by \textbf{S1}. The two settings are:

\begin{figure*}[h]
\captionsetup{width=0.3\textwidth}
     \centering
     \subfloat[subfigcapskip =50pt][\noindent \scriptsize Synthetic data-set with $L = 455$, $K = 210$ and $m = 7$. The rewards are Uniform around the means with a support of length $0.4$. Setting : $\mathbf{S1}$ \normalsize]
     {\includegraphics[width=0.32\linewidth]{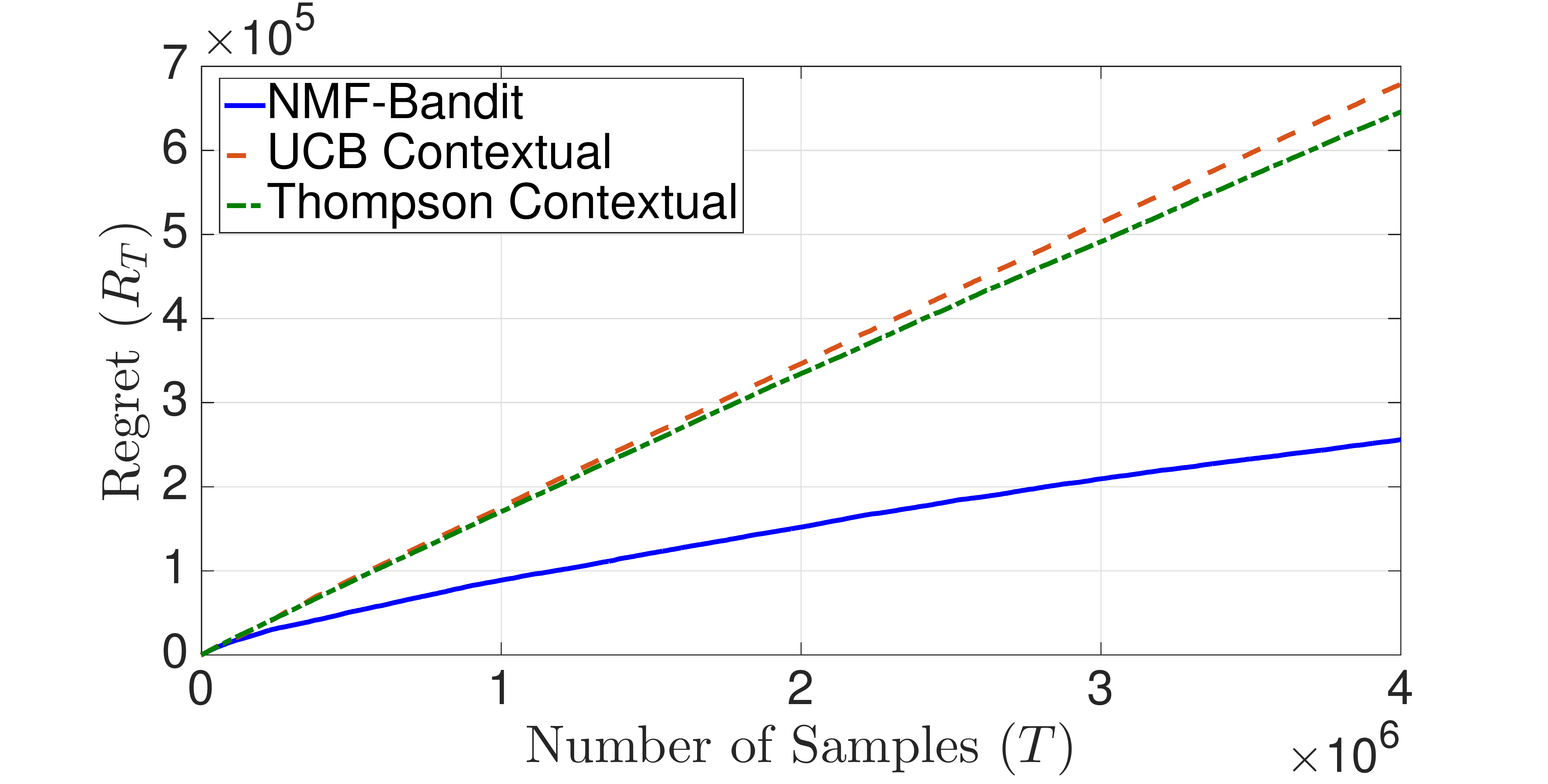}\label{fig:lowvar}}  \hfill
     \subfloat[subfigcapskip =50pt][\noindent \scriptsize Synthetic data-set with $L = 300$, $K = 145$ and $m = 3$; the rewards are Bernoulli with the given means. Setting : $\mathbf{S1}$ \normalsize]
     {\includegraphics[width=0.32\linewidth]{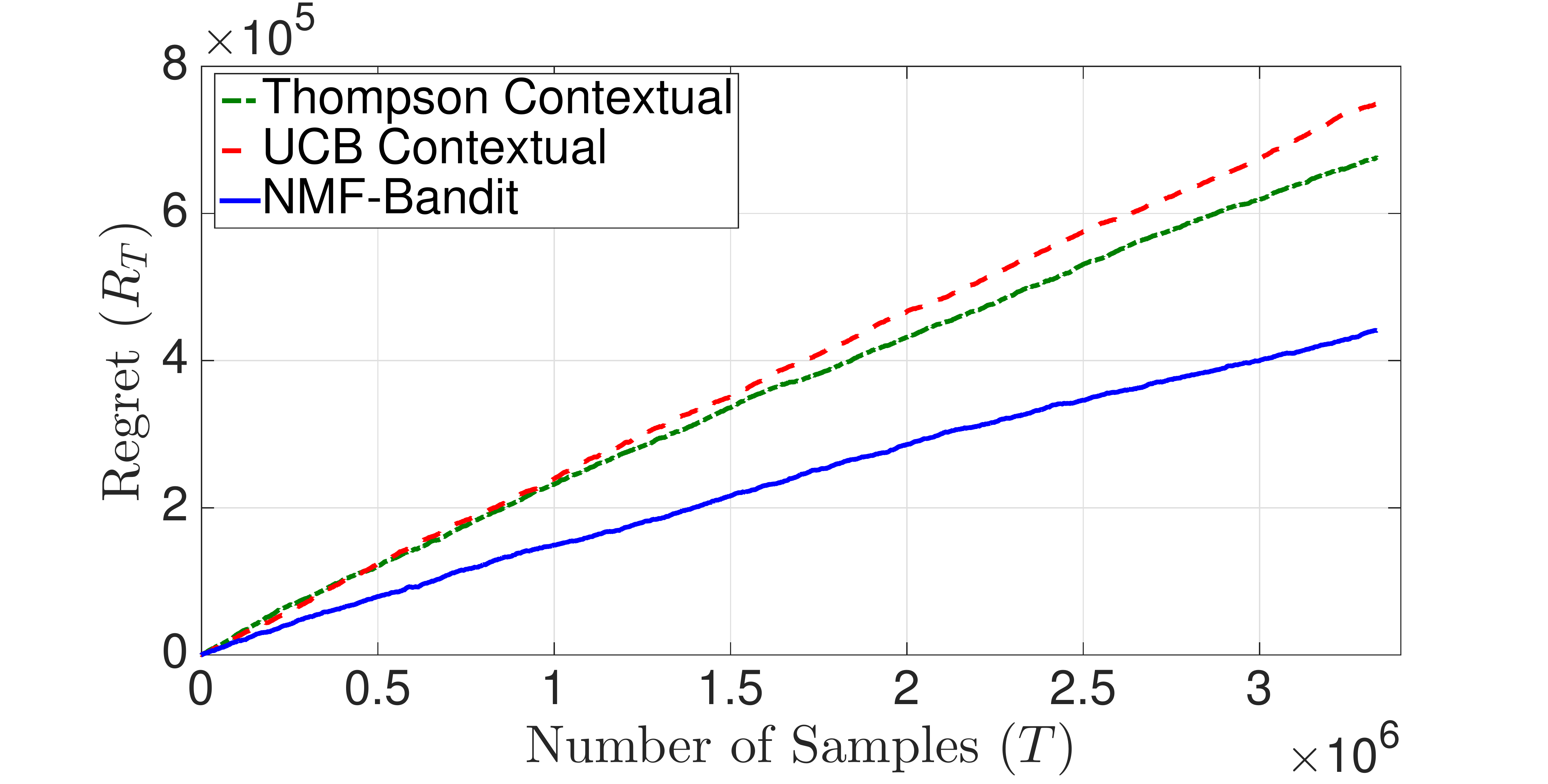}\label{fig:highvar}} \hfill
     \subfloat[subfigcapskip =50pt][\noindent \scriptsize A random subset of Movielens data-set with $L = 1600$, $K = 750$. The rewards are Uniform around the means with a support of length $2$. In our algorithm we set $m = 10, m' = 20$ and $\theta = 3$. Setting : $\mathbf{S1}$ \normalsize]
     {\includegraphics[width=0.32\linewidth]{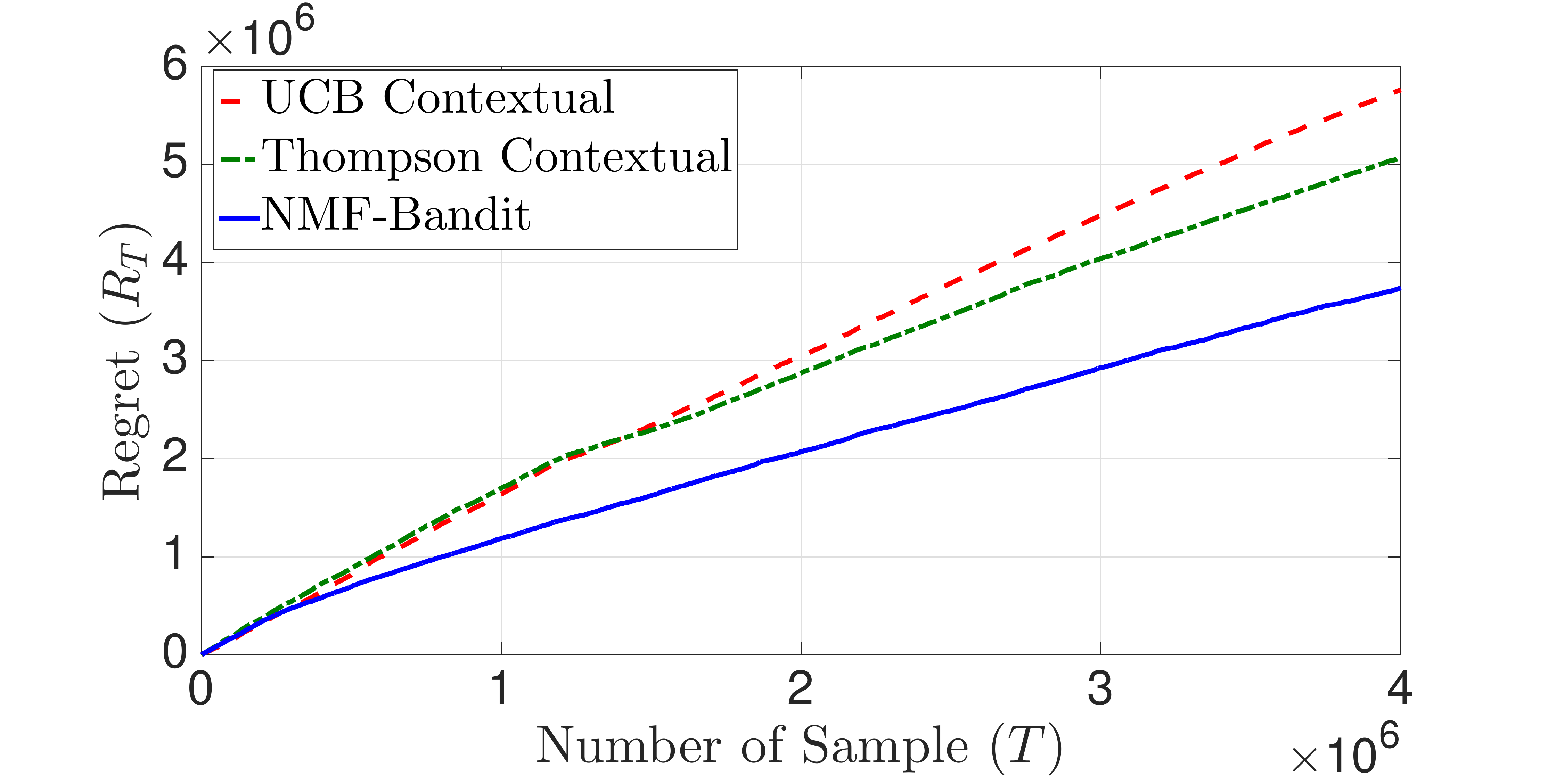} \label{fig:mlens}} \\
     \subfloat[subfigcapskip =50pt][\noindent \scriptsize A random subset of Book Crossing data-set with $L = 3000$, $K = 1450$ . The rewards are Uniform around the means with a support of length $2$. In our algorithm we set $m = 15, m' = 30$ and $\theta = 3$. Setting : $\mathbf{S1}$ \normalsize]
     {\includegraphics[width=0.32\linewidth]{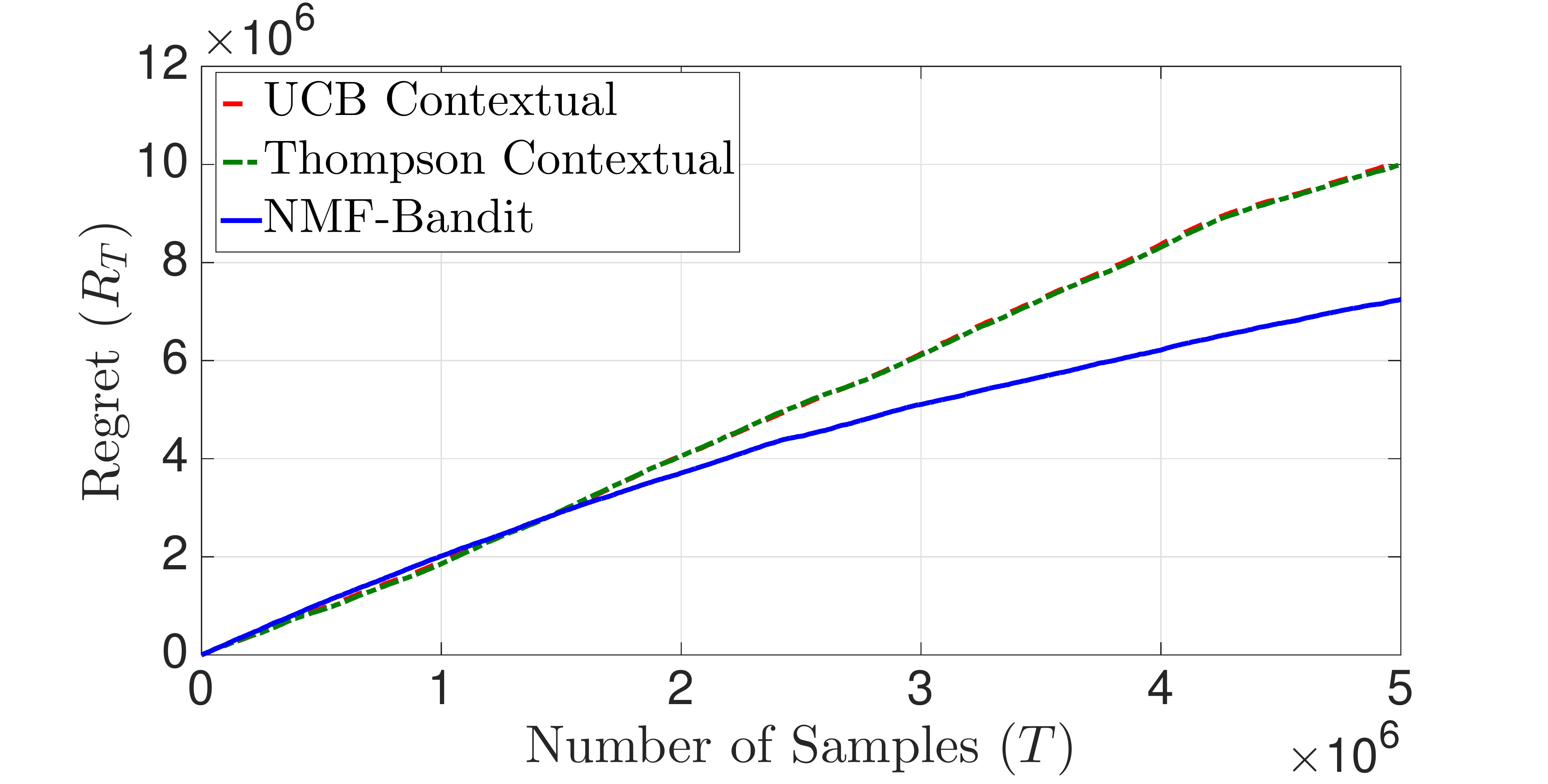} \label{fig:BC}} \hfill
     \subfloat[subfigcapskip =50pt][\noindent \scriptsize Synthetic data-set with $L = 90$, $K = 30$ and $m = 3$. The rewards are Uniform around the means with a support of length $0.4$. Setting : $\mathbf{S2}$ \normalsize]
     {\includegraphics[width=0.32\linewidth]{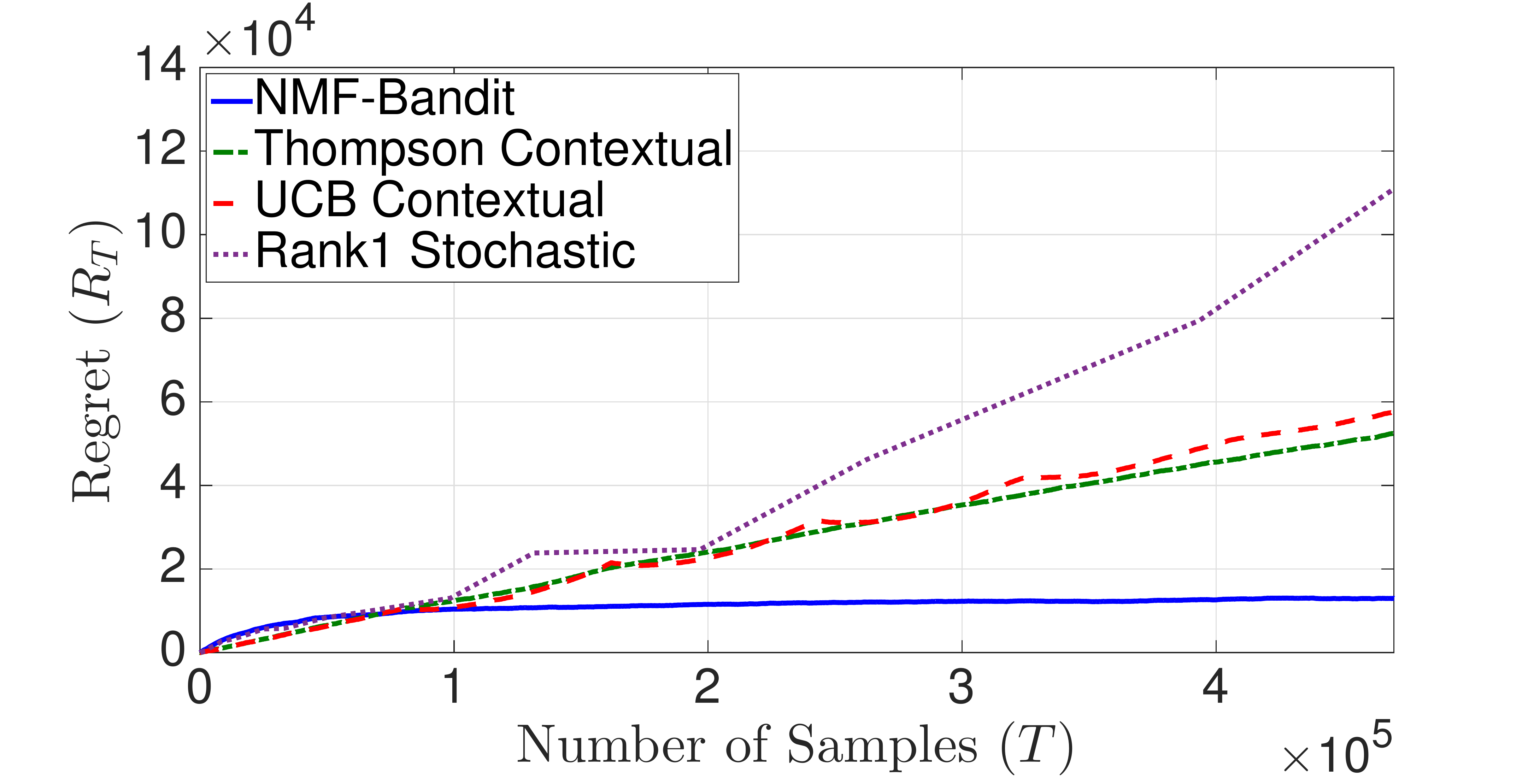} \label{fig:s2synth}} \hfill
     \subfloat[subfigcapskip =50pt][\noindent \scriptsize A random subset of Book Crossing data-set with $L = 1000$, $K = 450$. The rewards are Uniform around the means with a support of length $4$.  Setting : $\mathbf{S2}$ \normalsize]
     {\includegraphics[width=0.32\linewidth]{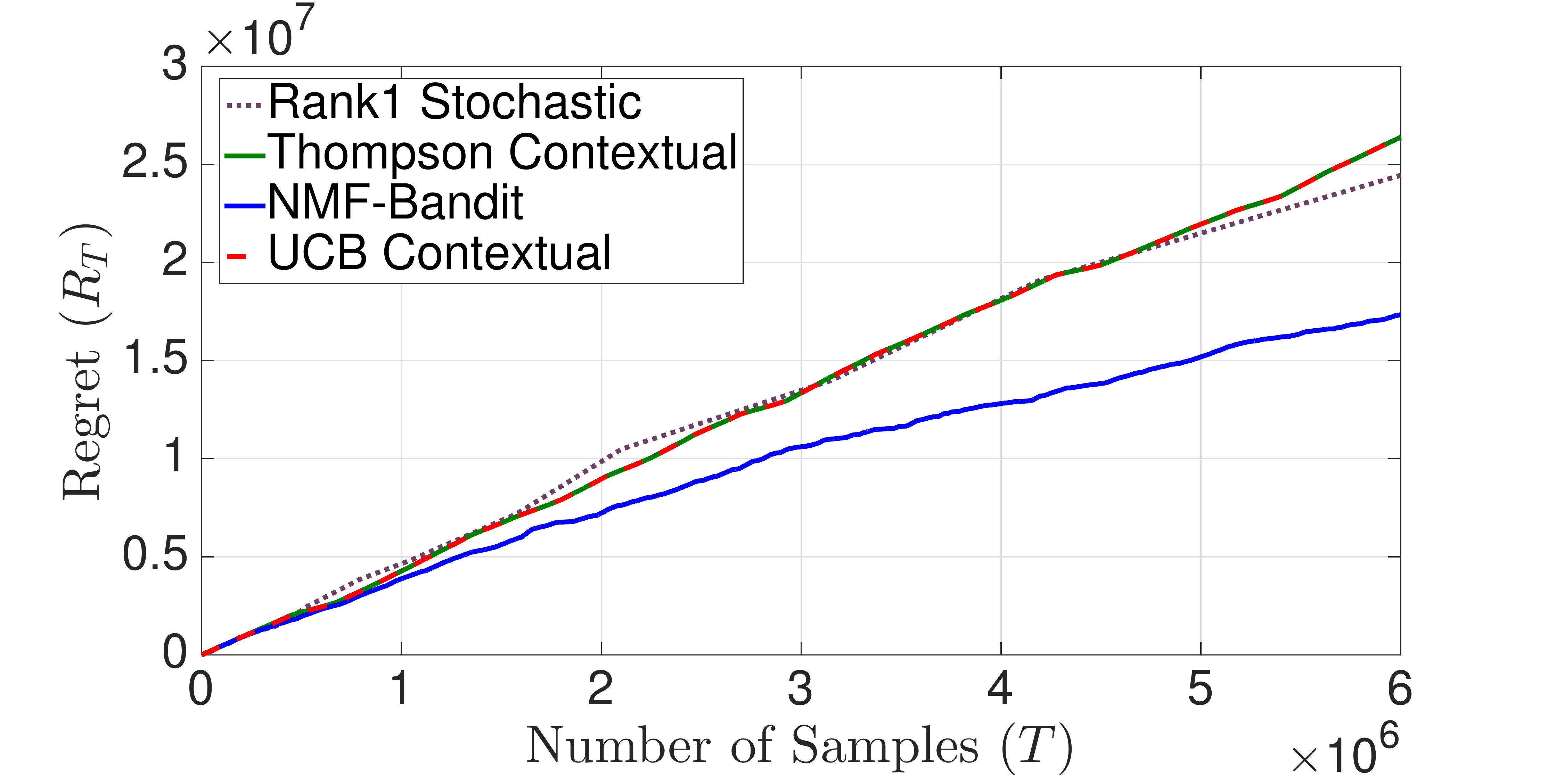} \label{fig:s2real}}
     \captionsetup{width=1\linewidth}
     \caption{Comparison of contextual versions of UCB-1,Thompson sampling and Rank 1 Stochastic bandits with Algorithm~\ref{alg:LCB} (NMF-Bandit) in \textbf{S1} and \textbf{S2} on real and synthetic data-sets.}
     \label{fig:comparem1}
\end{figure*}

\noindent  \textbf{S1} : The arrival of the contexts \textit{cannot} be controlled by the algorithm and the regret is w.r.t the best arm which is context \textit{dependent}. This is strongly motivated by the causal setting discussed with real world scenarios in Section \ref{intro}.

\noindent \textbf{S2} : This is in accordance with the model in~\cite{katariya2016stochastic}. The contexts and the arms \textit{both can} be chosen by the algorithm and the aim is to compare regret w.r.t the best arm out of \textit{all} $KL$ entries. 

\noindent {\bf Synthetic Data-Sets :  }  In order to generate the synthetic reward matrix $\mathbf{U}$, the parameter $L,K,m$ are chosen. The $L \times m$ matrix $\mathbf{A}$ is then generated by picking each row uniformly at random from the $m$-dimensional simplex. The $m \times K$ matrix $\mathbf{W}$ is generated with each entry uniformly generated in the interval $[0,1]$. We further corrupt $5$ \% of the entries in each row of $\mathbf{W}$ with completely arbitrary noise while ensuring that they still lie in $[0,1]$.  

In Figure~\ref{fig:lowvar},\ref{fig:highvar}, we compare our algorithm to UCB-1 and Thompson in \textbf{S1} under different values of problem dimensions. In Figure~\ref{fig:lowvar}, the rewards are uniform with means given by $\mathbf{U}$, while they are Bernoulli in  Figure~\ref{fig:highvar}. We observe that UCB-1, Thompson have linear regret as they do not get sufficient concentration for the  $L \times K$ mean parameters. However, our algorithm is able to enter the sub-linear regime much faster. We mention the choice of the parameters $\theta$ and $m'$ below the corresponding figures. It should be noted that our algorithm performs well even for values of the \textit{explore} parameter $\theta$, which are much lower than prescribed. In Figure~\ref{fig:s2synth} the experiments are performed under \textbf{S2}. We can see that our algorithm's regret is better compared to the others by a large margin, even though it has not been designed for this setting.

%

\noindent {\bf Real World Data-Sets :  }  We use the Movielens 1M~\cite{harper2016movielens} and the Book Crossing~\cite{ziegler2005improving} data-sets for our real world experiments. A subset of dimension $2000 \times 2000$ is chosen from the  Movielens 1M dataset, such that we have at least 20 ratings in each row and each column. Similarly a subset of $3000 \times 3000$ is chosen from the Book Crossing data-set with the same property. Both these partially incomplete rating matrices are then completed using the Python package \textit{fancyimpute}~\cite{hastie2015matrix} using the default settings. These completed matrices are used in place of the reward matrix $\mathbf{U}$ without any further modifications, and all the algorithms are completely agnostic to the process through which these matrices have been completed. The experiments have been performed in a setting where the rewards observed are uniform around the given means. The support of the uniform distributions has been specified below each figure. 

In Figure~\ref{fig:mlens} and \ref{fig:BC}, we compare our algorithm to UCB-1 and Thompson in \textbf{S1} on the MovieLens and Book Crossing data-set respectively. As before, our algorithm has superior performance. In Figure~\ref{fig:s2real}, we compare the algorithms on the Book Crossing data-set under \textbf{S2}.  NMF-Bandit outperforms all the other algorithms, even on the real datasets. 


\section{Conclusion}
In this paper we investigate a causal model of contextual bandits (as shown in Figure~\ref{fig:causalgraphs2}) with $L$ observed contexts and $K$ arms, where the observed context influences the reward through a latent confounder. The latent confounder is correlated with the observed context and lies in a lower dimensional space with only $m$ degrees of freedom. We identify that under this causal model, the reward matrix $\mathbf{U}$ naturally factorizes into non-negative factors $\mathbf{A}$ and $\mathbf{W}$. 

We propose a novel $\epsilon$-greedy algorithm (NMF-Bandit), which attains a regret guarantee of $O(L\mathrm{poly}(m,\log K)\log T /\Delta^2)$. Our guarantees are under statistical RIP like conditions on the non-negative factors. We also establish a lower bound of $O(Km\log T /\Delta)$ for our problem. To the best of our knowledge, this is the first achievable regret guarantee for online matrix completion with bandit feedback, when rank is greater than one. 

We validate our algorithm on real and synthetic datasets and show superior performance with respect to the baselines considered. This work opens up the prospect of investigating general causal models from a bandit perspective, where the goal is to control the regret of a target variable, when the algorithm can intervene only on some of the variables (\textit{limited interventional capacity}), while other variables (possibly \textit{latent}) can causally influence the reward. This is a natural setting and we expect that it will lead to an interesting research direction.
\FloatBarrier






\bibliographystyle{plain}
\bibliography{latent.bib}
\newpage
\appendix
\section{Appendix}
\subsection{Algorithmic Details}
\label{sec:algdetails}
We present a precise version of the algorithm described in Section~\ref{sec:algo} as Algorithm~\ref{alg:LCB}. 
For ease of exposition, we introduce the concept of matrix sampling, which is a notational tool to represent the sampled entries from different subsets of arms in a structured manner. 
\subsubsection{Matrix Sampling}
\label{sec:matrixsample}
Consider the $L \times K$ reward matrix $\mathbf{U}$. Consider a `sampling matrix' $\mathbf{G}$ with dimensions $K \times p$. Let $\{a_1,a_2 \ldots a_p\} \subset [K]$. In this work, we consider $\mathbf{G}$ only of the following form: $\mathbf{G}_{a_i,i} = 1, ~ \forall 1 \leq i \leq p $ and zero otherwise. Consider the product between a row $s$ of $\mathbf{U}$ and $\mathbf{G}$, i.e. $\mathbf{U}_{s,:} \mathbf{G}$. This selects the co-ordinates corresponding to $\{a_1 \ldots a_p\}$ in vector $\mathbf{U}_{s,:}$. Given a row $s$ (a context $s$) of $\mathbf{U}$, i.e. $\mathbf{U}_{s,:}\coloneqq \mathbf{u}[s]$, we describe how to obtain a random Bernoulli vector estimate $\mathbf{\hat{u}}[s]$ such that  $\mathbb{E} [\mathbf{\hat{u}}[s]] = \frac{1}{p} \mathbf{U}_{s,:}$ by sampling an arm as follows:
\begin{itemize}
\item Given that the context is $s$, sample a uniform random variable $\kappa$ with support $\{a_1 \ldots a_p \}$, which represents the arm to be pulled after observing the context.
\item Conditioned on $\kappa = k$, pull arm $k$ and observe the reward $Y_{k} \in \{0,1\}$. 
\item The random vector sample is then given by $\mathbf{\hat{u}}[s]_k = Y_\kappa \mathbf{e}_{\kappa}$. 
\end{itemize} 
Then we have $\mathbb{E}[\mathbf{\hat{u}}[s]_k] = \mathbb{E}[\mathbb{E}[Y_k|\kappa = k]] = \frac{1}{p}\mathbf{u}[s]_k$. In other words, whenever the context is $s$, we pull an arm uniformly at random from $\{a_1,a_2 \ldots a_p\}$ and the samples are collected in $\mathbf{\hat{u}}[s]$. 
\subsubsection{Arms to be sampled during \textit{explore}}
Before we present the pseudocode, we define the sampling matrices $\left\{ \mathbf{G}(0), \mathbf{G}(1), \cdots, \mathbf{G}(l+1)\right\}$. Recall that any subset of arms can be encoded in a sampling matrix. $\mathbf{G}(0)$ corresponds to the subset $S$ in Step 1 of \textit{explore} stated in Section~\ref{sec:algo}. For ease of reference, we restate the sets relevant to the context specific sampling procedure in Step 2 of \textit{explore}. $\mathbf{G}(i)$ corresponds to the subset $R(s_t)$ is $s_t \in S(i)$. Let $l = \lfloor K/m \rfloor$ and $r = K\mathrm{mod}(m)$. A set $R \subset [L]$ of contexts is sampled at random, such that $\lvert R \rvert = 2(l+1)m'$ at the onset of the algorithm. We partition $R$ into $l+1$ contiguous subsets $\left\{S(1),S(2),...,S(l+1) \right\}$ of size $2m'$ each. The elements of the set $S(j)$ will be denoted as $S(j) = \{s_1(j), s_2(j) \cdots, s_{2m'}(j) \}$. In Step 2 of \textit{explore}, if $s_t \in S(i)$, then $R(s_t) = \{(i-1)m,(i-1)m+1,\cdots \max (im-1,K) \}$. If $s_t \notin S(i)$ for all $i \in [l+1]$, then the algorithm is allowed to pull any arm at random, and these samples are ignored.
\begin{enumerate}
\item $\mathbf{G}(0)$: An $ K \times 2m'$ random matrix formed as follows: An $2m'$ subset $a_1,a_2 \ldots a_{2m'} \subset [K]$ is chosen randomly uniformly among all $2m'$-subsets of $[K]$ and $\mathbf{G}(0)_{a_i,i}=1, ~ \forall 1 \leq i \leq 2m' $ and all other entries are $0$. \\
\item $\mathbf{G}(i)$: An $K \times m$ matrix such that,
\begin{equation*}
G(i)_{kj} = \left \{
\begin{array}{cc}
    1, & \text{if } k = (i-1)m + j \text{ for } j \in \{1,\cdots ,m\} \\
    0, & \text{otherwise}
\end{array}  
\right \}
\end{equation*}
when $i \in \{ 1,2, \cdots, l\}$.  
\item $\mathbf{G}(l+1)$: An $K \times r$ matrix defined as follows:
\begin{equation*}
G(l+1)_{kj} = \left \{
\begin{array}{cc}
    1, & \text{if } k = (lm + j) \text{ for } j \in \{1,\cdots r\} \\
    0, & \text{otherwise}
\end{array}  
\right \}
\end{equation*}
\end{enumerate} 
In words, $\mathbf{G}(i)$ for $i\in[l]$ is the $K\times m$ matrix which has an identity matrix $I_{m \times m}$ embedded between rows $(i-1)m$ and $im-1$, and is zero everywhere else.
\subsubsection{ Representation of the collected Samples } 
In what follows, let the mean of samples collected through $\mathbf{G}(0)$ till time $t$ be collected in a $L \times 2m'$ matrix $\hat{\mathbf{F'}}(t)$ such that  $\EE \left[ \hat{\mathbf{F'}}(t)\right] = (1/2m')\mathbf{F} = (1/2m')\mathbf{U}\mathbf{G}(0)$ as detailed in Section~\ref{sec:matrixsample}. Let $\hat{\mathbf{F}}(t) = 2m'\hat{\mathbf{F'}}(t)$. Let the samples collected from $\mathbf{G}(i)$ be stored in a $2m' \times m$ matrix $\hat{\mathbf{M'}_i}(t)$ such that $\EE \left[ \hat{\mathbf{M'}}_i(t) \right] = \frac{1}{m}\mathbf{A}_{S(i),:}\mathbf{W}\mathbf{G}(i)$ for all $i \in \{1,2,...,l+1\}$. Let $\hat{\mathbf{M}}_i(t) = m\hat{\mathbf{M'}}_i(t)$ be the scaled version.

\subsubsection{Pseudocode}
\begin{algorithm*}
  \caption{NMF-Bandit - An $\epsilon$-greedy algorithm for Latent Contextual Bandits}
  \begin{algorithmic}[1]
    \State At time $t$, 
    \State Observe context $S_t = s_t$
    \State Let $\mathsf{E}(t) \sim \mathrm{Ber}(\epsilon_t)$
    \If { $\mathsf{E}(t) = 1$}
    \State \textit{Explore:} Let $H_t \sim  \left \{
\begin{array}{cc}
    \mathrm{Ber}\left(\frac{2m'}{r+2m'}\right), & \text{if } s_t \in S(l+1)  \\
    \mathrm{Ber}\left(\frac{2m'}{m+2m'}\right), & \text{otherwise}
\end{array}  
\right \}.$ 
    \State If $H_t=1$ sample an arm according to the matrix sampling technique applied to matrix $\mathbf{G}(0)$ and update $\hat{\mathbf{F}}(t)$. 
    \State If $H_t=0$ sample an arm according to the matrix sampling technique applied to matrix $\mathbf{G}(i)$ if $s_t \in S(i)$ for $i \in \{ 1,2,\cdots,l+1\}$ and update $\hat{\mathbf{M}}_i(t)$. If $s_t$ is not in any of these sets then choose an arm at random. 
    \Else
    \State \textit{Exploit:}
    \State Let us compute,
    \begin{align*}
    \hat{\mathbf{W}}(t) &= \mathrm{Hottopix}(\mathbf{F}(t),m,2m' \gamma(t)). \\
    \hat{\mathbf{A}}(t) &= \argmin_{\mathbf{Z} \geq \mathbf{0}, \mathrm{rowsum}(\mathbf{Z}) = \mathbf{1}} \norm{\mathbf{F}(t) - \mathbf{Z}\hat{\mathbf{W}}(t)}_{\infty,1}.
\end{align*}     
    \State Let $\hat{\mathbf{W}}(t) \in \mathbb{R}^{m \times K}$ be such that, 
 	\begin{align*}
 	\hat{\mathbf{W}}(t)_{:,(i-1)m:im - 1} & = \argmin_{\mathbf{X}_{m \times m}} \norm{\hat{\mathbf{A}}(t)_{S(i),:}\mathbf{X} - \hat{\mathbf{M}}_i(t)}_{2}, \mbox{ } \forall i \in \{1,2,..,l \} \\
 	\hat{\mathbf{W}}(t)_{:,lm:K} & = \argmin_{\mathbf{X}_{m \times r+1}} \norm{\hat{\mathbf{A}}(t)_{S(l),:}\mathbf{X} - \hat{\mathbf{M}}_{l+1}(t)}_{2}
\end{align*} 	
\State Compute $\hat{\mathbf{U}}(t)  = \hat{\mathbf{A}}(t)\hat{\mathbf{W}}(t)$. Play the arm $a_t$ such that,
	\begin{align*}
	a_t = \argmax _{a} \hat{\mathbf{U}}(t) _{s_t,a}
	\end{align*}
  	    
    \EndIf 
      \end{algorithmic}
    \label{alg:LCB}
\end{algorithm*}

We present a detailed pseudo-code of our algorithm as Algorithm~\ref{alg:LCB}. For the sake of completeness we include the robust version of the Hottopix algorithm~\cite{gillis2014fast} which is used as a sub-program in Algorithm~\ref{alg:LCB}. The following LP is fundamental to the $\mathrm{Hottopix}$ algorithm,
\begin{align*}
& \min_{\mathbf{C} \in \mathbb{R}^{f \times n}_+} \mathbf{p}^T\mathrm{diag}(\mathbf{C}) \numberthis \label{eq:lp}\\
& \text{ s.t. } \norm{\tilde{\mathbf{X}} - \mathbf{C}\tilde{\mathbf{X}}}_{\infty,1} \leq 2\epsilon \\
& \text{ and } C_{ii} \leq 1, \mbox{   } C_{ji} \leq C_{ii} \mbox{ } \forall i,j \in [L] 
\end{align*}
where $\mathbf{p}$ is a vector with distinct positive values. 
\begin{algorithm*}
  \caption{$\mathrm{Hottopix}(\tilde{\mathbf{X}},m,\epsilon)$}
  \begin{algorithmic}[1]
    \State {\bf Input :} $\tilde{\mathbf{X}}$ such that $\tilde{\mathbf{X}} = \mathbf{A}\mathbf{W} + \mathbf{N}$ , where $\mathbf{A} \in [0,1]^{L \times m}$ and $\norm{A_{i,:}}_1 = 1$ for all $i \in [L]$, $\mathbf{W} \in \mathbb{R}_+^{m \times 2m'}$ and $\norm{\mathbf{N}}_{\infty,1} \leq \epsilon$. 
    \State {\bf Output :} $\hat{\mathbf{W}}$ such that $\hat{\mathbf{W}} \sim \mathbf{W}$. 
    \State Compute an optimal solution $\mathbf{C}^*$ to (\ref{eq:lp}).  
    \State Let $\mathcal{K}$ denote the set of indices $i$ for which $C^*_{ii} \geq \frac{1}{2}$. 
    \State Set $\hat{\mathbf{W}} = \tilde{X}_{\mathcal{K},:}$. 
  \end{algorithmic}
    \label{alg:hott}
\end{algorithm*}
\FloatBarrier

\subsection{Theoretical Insights}
\label{sec:thinsight}

 Below, we discuss some of the key
challenges in the theoretical analysis.

{\bf Noise Guarantees for samples used in NMF:} Matrix completion
algorithms that work under the incoherence assumptions require the
noise in each element of the matrix to be $O(1/K)$ in order to provide
$l_{\infty}$-norm guarantees on the recovered matrix~\cite{hardt2014fast}.
In order to ensure such noise guarantees, we require a very large
number of samples in order for estimates to concentrate. This in turn
increases bandit exploration which implies that regret scales as
$O(LK\log(T))$. To avoid this, we follow a different route. In Step 1 of the
\textit{explore} phase, the NMF-Bandit algorithm only samples from a small
subset of arms denoted by $S$. By leveraging the $\ell_1$-WStRIP
property of $\mathbf{W}$, we can ensure that NMF on these samples
(which are basically a noisy version of $\mathbf{U}_{:,S}$ ) gives us
a good estimate of $\mathbf{A}$ at time $t$; this estimate is denoted by
$\hat{\mathbf{A}}(t)$. We prove this statement formally in
Lemma~\ref{lem:nmf}. Given that we sample only from a small subset of
arms in the first step of \textit{explore}, in Lemma~\ref{lem:Finf} we
show that the samples concentrate sharply enough.

{\bf Ensuring enough linear equations to recover $\mathbf{W}$:}
Recall that the reward matrix has the structure
$\mathbf{U} = \mathbf{AW}$. Therefore, an initial approach would be 
to use the current estimate of  $\mathbf{A}$ along with samples of the
rewards, and directly recover
$\mathbf{W}$. This however will not work due to lack of
concentrations. First, the estimate of $\mathbf{A}$ in the early
stages will be too noisy to provide sharp estimates about the location
of the extreme points aka the latent contexts. Even if we knew the
identities of the observed contexts that correspond to ``pure''
latent contexts (extreme points of the affine space corresponding to
the observed contexts), most observed contexts will not correspond to
these extreme points -- thus, a large number of samples will be
wasted, again leading to poor concentrations. Second, if one decides
to sample the entries in $\mathbf{U}$ at random, the concentration of
the entries would be too weak. As before, these weak concentrations
will imply $O(LK\log(T))$ regret.

Instead, we design the context dependent sets of arms to pull in Step 2
of the \textit{explore} phase, such that we get enough independent
linear equations to recover $\mathbf{W}.$ The key is to have a small
number of arms to sample per observed contexts, but the small number
of arms differ across observed contexts.
In this case, we show that by leveraging the $\ell_2$-WStRIP property of
$\mathbf{A}$ we can get a good estimate of $\mathbf{W}$, denoted by
$\hat{\mathbf{W}}(t)$ even in the presence of sampling noise. Since
we sample from a small subset of arms for each observed
context, in Lemma~\ref{lem:minf} we can ensure that we have sharp
concentrations.

{\bf Scheduling the optimal arm during \textit{exploit}:} The
$l_{\infty}$-norm bounds on the errors in $\hat{\mathbf{A}}(t)$ and
$\hat{\mathbf{W}}(t)$, imply that
$\norm{\hat{\mathbf{U}}(t) - \mathbf{U}}_{\infty,\infty} < \Delta/2$
with probability at least $ 1 - O(\frac{Lm'}{t})$ provided $\epsilon _t$
is sufficiently big (see proof of Theorem~\ref{thm:mainub}). Here
$\Delta = \min_{s \in [L]} (u^*(s) - \max_{k \neq k^*(s)}U_{s,k})$.
This essentially implies that the correct arm is pulled at time $t$
w.h.p if the algorithm decides to \textit{exploit}.

\subsection{Description of Generative Models for matrices $\mathbf{W}$ and $\mathbf{A}$}
\label{sec:detmodels}
The model for $\mathbf{W}$ and $\mathbf{A}$ are both very similar with deterministic and random parts. The technical description of the model given below is complex due to the following two reasons:
\begin{enumerate}[itemsep=-3pt]
\item \textbf{Fact 1:} Rows of $\mathbf{A}$ must sum to $1$. 
\item \textbf{Fact 2:} The rows of $\mathbf{W}$ shifted by an arbitrary vector $\mathbf{m} \in \mathbb{R}^{1 \times K}$ does not affect the NMF algorithms employed. The setting is invariant to such a shift.
\end{enumerate}
\begin{enumerate}
\item \textit{Random+Deterministic Composition}:
\begin{enumerate}
 \item We assume that columns $\mathbf{W}_{:,D}$ corresponding to the column index set $D \subseteq [K], ~\lvert D \rvert \leq K/(32m)$ is arbitrary and deterministic. $0 \leq W_{i,j} \leq 1,~ j \in D$. The maximum entry in every row of $\mathbf{W}$ is assumed to be contained in the deterministic part. 
 \item Similarly, $\mathbf{A}_{E,:}$ where $E \subseteq [L]$ is arbitrary and deterministic. Let $\lvert E \rvert \leq \rho L $. $\rho=1/18$. Row sum of every row of $\mathbf{A}_{E,:}$ is $1$. In order to ensure separability~\cite{recht2012factoring} we assume that there is a subset $M \subseteq E: \lvert M \rvert=m$ such that $\mathbf{A}_{M,:} = \mathbf{I}_{m \times m}$. For all $i \in E-M$, $0 \leq A_{ij} \leq \gamma < 1$. 
\end{enumerate}

 \item \textit{ Bounded randomness in the random part:} 
     \begin{equation}
         \mathbf{W}_{:,D^c} =  \mathbf{1}* \mathbf{m}^T+ \mathbf{R}_{:,D^c} + \mathbf{\tilde{W}}_{:,D^c} 
     \end{equation}
    \begin{enumerate} 
   \item $\left(i,j \right)$-th entry of $\mathbf{\tilde{W}}_{:,D^c}$ is an independent mean zero sub-gaussian entry with variance $q$, and bounded support and sub-gaussian parameter $c(q)$. $\mathbf{m} \in  \mathbb{R}^{\lvert D^c \rvert \times 1}$ is an arbitrary deterministic vector \footnote {This is introduced to respect Fact $2$ in Section \ref{sec:detmodels}} . 
   \item $\mathbf{R}_{:,D^c}$ is a deterministic perturbation matrix satisfying $\lVert \mathbf{R}_{:,j} \rVert_{2} \leq \frac{1}{5}, ~\forall j \in D^c$. The support parameters for $\mathbf{\tilde{W}}_{:,D^c}$, $\mathbf{m}$ and $\mathbf{R}_{:,D^c}$ are chosen such that $0 \leq W_{i,j} \leq 1 \mathrm{~a.s.~}, ~ \forall j \in D^c$
   \end{enumerate}

   $\mathbf{A}_{E^c,:}$ is a matrix which is a row-normalized version of another random matrix $\mathbf{\tilde{A}}$. We first describe the random model on the $\lvert E^c \rvert \times m$ matrix $\mathbf{\tilde{A}}$.  Like in the case for model of $\mathbf{W}$,
 \begin{equation}
    \mathbf{\tilde{A}} =  \mathbf{N} + \mathbf{\hat{A}} 
 \end{equation}
 \begin{enumerate}
\item $\mathbf{\hat{A}}$ is a matrix with independent mean zero sub-gaussian entries each with variance $q$, and bounded support and sub-gaussian parameter $c(q)$. 
\item We denote the matrix of means by $\mathbf{N}$ consisting of the parameters $n_{ij}$. The $\ell_2$ norm of every row of $\mathbf{N}$ is at most $\frac{1}{5}$. The support, sub-gaussian parameter and the matrix of means $\mathbf{N}$ are chosen such that $1/m \leq \tilde{A}_{ij} \leq \gamma <1 ~\mathrm{a.s}$. The stricter condition (in the lower bound) ensures that after normalization by the row sum, $A_{ij} \leq \gamma < 1,~ i \in E^c$.
\end{enumerate}
 \end{enumerate}

\subsection{Projection onto a Low Dimensional Space}
\label{sec:projall}
In this section, we will prove some properties of the matrix $\mathbf{F}= \mathbf{U}\mathbf{G}(0)=\mathbf{A}\mathbf{W}\mathbf{G}(0)$ where $\mathbf{G}(0)$ is a $K \times 2m'$ as defined in Section \ref{sec:matrixsample}. From the definition in Section \ref{sec:contexts}, $\mathbf{A}$ contains a $\mathbf{I}_{m \times m}$ sub-matrix corresponding to the rows in ${\cal Z}$. Further, the row sum of every row of $\mathbf{A}$ is $1$. This means that the rows of $\mathbf{U}$ consists of points in the convex hull of extreme points, i.e. the rows of $\mathbf{W}$, together with the extreme points themselves.

The extreme points in $\mathbf{W}$ are mapped to extreme points in $\mathbf{W}\mathbf{G}(0)$. We also show that the new set of extreme points $\mathbf{W}\mathbf{G}(0)$ also satisfy what is called the simplical property when $\mathbf{W}$ satisfies the assumptions in Section \ref{sec:detmodels}.

When the entries in $\mathbf{W}$ are random and independent bounded random variables as in Section \ref{sec:models}, we show that $\ell_1$ distance of any non-zero vector $\mathbf{a}$ such that $\mathbf{a}^T\mathbf{1}=0$ is preserved under the map $\mathbf{a}^T \mathbf{W}\mathbf{G}(0)$with high probability over $\mathbf{W}$ for any fixed $\mathbf{G}(0)$. 
 We need some results relating to sub-gaussianity of the matrix $\mathbf{W}$ which we deal with in the next subsection.

\subsection{Sub-gaussianity of a matrix with bounded i.i.d random entries}

\begin{definition}
 \cite{foucart2013mathematical}   A random variable $X$ is sub-gaussian with parameter $c>0$ if $\mathbb{E}[\exp(tX)] \leq \exp(-c^2t^2), ~\forall t \in \mathbb{R}$. 
\end{definition}

\begin{definition}
  \cite{foucart2013mathematical} A random vector $\mathbf{Y} \in \mathbb{R}^n$ is isotropic if $\mathbb{E}[(\mathbf{Y}^T\mathbf{x})^2] = \mathbb{E}[\mathbf{x}^T\mathbf{x}],~ \forall \mathbf{x} \in \mathbb{R}^n$. It is sub-gaussian with parameter $c$ if the scalar random variable $\mathbf{Y}^T \mathbf{x}$ is sub-gaussian with parameter $c$ for all $\mathbf{x} \in \mathbb{R}^n: \lVert \mathbf{x} \rVert_2=1$, i.e. $\mathbb{E}[ \exp(t (\mathbf{Y}^T \mathbf{x}))] \leq \exp(-ct^2), \forall t \in \mathbb{R}, ~\forall  \lVert \mathbf{x} \rVert_2=1 $.
\end{definition}

\begin{lemma}\label{subgauss}
 \cite{foucart2013mathematical},\cite{stromberg1994probability}  Consider a random variable $X$ such that $\mathbb{E}[X]=0,~\mathbb{E}[X^2] =1,~|X| \leq b ~ \mathbf{a.s}$ for some constant $b>0$. Then, $X$ is sub-gaussian with parameter $\frac{b^2}{2}$.  Consider a random vector $\mathbf{Y} \in \mathbb{R}^n$ where each entry is drawn i.i.d from a mean zero, unit variance and a sub-gaussian distribution with parameter $c$. Then $\mathbf{Y}$ is a sub-gaussian isotropic vector with the same sub-gaussian parameter $c$
\end{lemma}
\textbf{Remark:} The first part is from Theorem $9.9$ in \cite{stromberg1994probability}  while the second part is from Lemma $9.7$ from \cite{foucart2013mathematical}.

\begin{lemma} \label{lemmasuseful}
 \cite{foucart2013mathematical} Let $\mathbf{P}$ and $\mathbf{Q}$ be two matrices of the same dimensions. Let $\sigma_{\mathrm{min}}$ and $\sigma_{\mathrm{max}}$ be the largest and smallest singular values of a matrix respectively. Then,
      \begin{equation}
       \lvert \sigma_{\mathrm{min} } \left(\mathbf{P}) -\sigma_{\mathrm{min}} (\mathbf{Q} \right) \rvert \leq \sigma_{\mathrm{max}} \left(\mathbf{P}-\mathbf{Q} \right)
    \end{equation}
Let $\mathbf{P} \in \mathbb{R}^{p \times q}$ where $p \geq q$. Then,
   \begin{equation}
     \sigma_{\mathrm{max}} (\mathbf{P}^T \mathbf{P} - \mathbf{I}_{n \times n}) \leq \delta \Rightarrow \sigma_{\mathrm{min}} (P) \geq \sqrt(1-\delta)
   \end{equation}
       
\end{lemma}

\begin{lemma}\label{singlower}
  \cite{foucart2013mathematical} Consider an $m \times s$ matrix $\mathbf{P}$ with every row being a random independent sub-gaussian isotropic vector with sub-gaussian parameter $c$. Let $m>s$, then: 
  \begin{align}
     \mathrm{Pr} \left( \sigma_{\mathrm{max}} \left( \frac{1}{m}\mathbf{P}^T \mathbf{P} - \mathbf{I}_{s \times s} \right) \geq \delta \right) & \leq  2\exp ( - \frac{3\tilde{c}}{4} \delta^2 m +\frac{7s}{2})
  \end{align}
  Further,
   \begin{align} 
    &\mathrm{Pr}\left( \sigma_{s} (\mathbf{P}) \leq \sqrt{m} \sqrt{(1 -\delta)} \right) \\
    & \leq  \mathrm{Pr} \left( \sigma_{\mathrm{max}} \left( \frac{1}{m}\mathbf{P}^T \mathbf{P} - \mathbf{I}_{s \times s} \right) \geq \delta \right)   \nonumber \\
     \hfill & \leq 2\exp ( - \frac{3\tilde{c}}{2}\delta^2 m +\frac{7s}{2})
   \end{align}
   Here, $\tilde{c}$ is a constant that depends only on the sub-gaussian parameter $c$.
\end{lemma}
\textbf{Remark:} The first result follows from equation $(9.15)$ in \cite{foucart2013mathematical} and also from combining Lemma $9.8$ and Lemma $9.9$ in \cite{foucart2013mathematical}. The second follows from applying Lemma \ref{lemmasuseful}


\begin{definition}[\cite{recht2012factoring}]\label{def:simplical}
 Let us consider a matrix $\mathbf{M}$ which is $p \times q$ where $p \leq q$. Let $\mathbf{m}_i \in \mathbb{R}^{1 \times p}$ be the $i$-th row of the matrix $\mathbf{M}$. The matrix $M$ is $\alpha$-simplical if $\min \limits_{i \in \{1\cdots p\}} \min \limits_{\mathbf{x} \in \mathrm{conv}(\{\mathbf{m}_1 \cdots \mathbf{m}_p\} \setminus \mathbf{m}_i)} \lVert \mathbf{m}_i - \mathbf{x} \rVert_1 \geq \alpha$. In other words, every row is at least $\alpha$ far away in $\ell_1$ distance from the convex hull of other points.
\end{definition}

   
 \subsection{Results regarding sub-matrices of $\mathbf{W}$}  
   
The following results hold for $\mathbf{W}\mathbf{G}(0))$ since $\mathbf{W}\mathbf{G}(0) = \mathbf{W}_{:,S}$ when $S = \{a_1 \ldots a_{m'} \}$ is the set of column indices associated with $\mathbf{G}(0)$ as in Section \ref{sec:algdetails}. 
   
\begin{theorem} \label{resrandW}

   Let $\mathbf{W}$ follow the random generative model in Section \ref{sec:models}. Let $S \subseteq D^c$. Let $|S| = m' \geq \frac{512 }{21\tilde{c}} m \log (eK) $,
   \begin{equation}
    \psi_{m} (\mathbf{W}_{:,S}) \geq \left( \frac{11}{20} \right) \sqrt{m'} 
    \end{equation}
     with probability at least $1-\frac{2}{K^{7m/2}}$ over the randomness in $\mathbf{W}$. Here, $\tilde{c}$ is a constant that depends on the sub-gaussian parameter $c(q)$ of the distributions in the generative model in Section \ref{sec:models}.
 \end{theorem}    
 \begin{proof}
 According to the random generative model for $\mathbf{W}$ in Section \ref{sec:models}, $\mathbf{W}_S=\mathbf{\tilde{W}}_{:,S}+\mathbf{1}\mathbf{m}^T_S+\mathbf{R}_S$. Here, $\mathbf{\tilde{W}}_{:,S}$ has sub-gaussian entries with parameter $c(q)$, since by Lemma \ref{subgauss}, all bounded random variables on support $[-1,1]$ with zero mean are sub-gaussian and their sub-gaussian parameter depends on the variance. Let $\mathbf{m}_{S}$ refer to the vector restricted to co-ordinate in $S$. 
 Applying Lemma \ref{singlower} to the sub-gaussian matrix ($m' \times m$) $\mathbf{\tilde{W}_{:,S}}$ with $m' \geq \frac{512}{21\tilde{c}} m \log (eK) $ and setting $\delta =7/16$, we have:
  \begin{align*}
    \mathrm{Pr} \left( \sigma_m ( \mathbf{\tilde{W}_{:,S}}^T ) \leq \frac{3}{4} \sqrt{m'} \right)  &\leq 2 \exp( - \frac{7}{2}m \log (K) ) \\
    &\leq 2 K^{-7m/2}.
      \end{align*}
     Now, applying Lemma \ref{lemmasuseful}, we have:
      \begin{align*}
       \lvert \sigma_m (\mathbf{R}_{:,S}+\mathbf{\tilde{W}}_{:,S}) -  \sigma_m (\mathbf{\tilde{W}}_{:,S}) \rvert  &\leq \sigma_{max} (\mathbf{R}_{:,S}) \\
       &\leq \Vert \mathbf{R}_{:,S} \rVert_{F} \\
       &\leq \frac{1}{5} \sqrt{m'}
      \end{align*}
      Combining the above two equations, we have: 
       \begin{align*}
    &\mathrm{Pr} \left( \sigma_m ( \mathbf{\tilde{W}}_{:,S} +  \mathbf{R}_{:,S}  ) \leq \left(\frac{3}{4} -\frac{1}{5} \right) \sqrt{m'} \right)  \\
    &\leq 2 \exp( - \frac{7}{2}m \log (K) ) \leq 2 K^{-7m/2}.
      \end{align*}
      For any fixed set of size $S=m'$, We have the following chain: 
      \begin{align}
      &\inf \limits_{\mathbf{a}\neq 0:\mathbf{a}^T\mathbf{1}=0} \frac{\lVert \mathbf{a}^T \mathbf{W}_{:,S} \rVert_2}{\lVert \mathbf{a}  \rVert_2} \\
      &=  \inf \limits_{\mathbf{a}\neq 0:\mathbf{a}^T\mathbf{1}=0} \frac{\lVert \mathbf{a}^T (\mathbf{1}\mathbf{m}^T_S+\mathbf{\tilde{W}}_{:,S}+\mathbf{R}_{:,S} )\rVert_2}{\lVert \mathbf{a}  \rVert_2} \nonumber \\
        \hfill &\underset{(\mathbf{a}^T\mathbf{1}=0)}= \inf \limits_{\mathbf{a}\neq 0::\mathbf{a}^T\mathbf{1}=0} \frac{\lVert \mathbf{a}^T (\mathbf{\tilde{W}}_{:,S}+\mathbf{R}_{:,S} )\rVert_2}{\lVert \mathbf{a}  \rVert_2}  \nonumber \\
       \hfill  & \geq \sigma_m (\mathbf{R}_{:,S}+\mathbf{\tilde{W}}_{:,S})        
      \end{align}

 \end{proof}   
   
 \begin{theorem}\label{proj1}
  Consider a matrix $\mathbf{W}$ with the generative model in Section \ref{sec:models}. Let $m' \geq \frac{512}{21\tilde{c}} m \log (eK)$. For any fixed set $S$ of size $2m'$ such that $S_1 = S \bigcap D,~ \lvert S_1 \rvert \leq \frac{2m'}{16m}$ we have: 
   \begin{equation}
   \psi_m^1(\mathbf{W}_{:,S}) =\inf \limits_{\mathbf{a} \neq 0: \mathbf{a}^T\mathbf{1}=0} \frac{\lVert \mathbf{a}^T \mathbf{W}_{:,S} \rVert_1}{\lVert \mathbf{a} \rVert_1} \geq  \left(\frac{13}{60} \right) \frac{\sqrt{15m'}}{\sqrt{8m}} 
      \end{equation}
   with probability at least $1-2K^{-7m/2}$ over the randomness in $\mathbf{W}$. Further, rows of $\mathbf{W}_{:,S}$ is $ \psi_{m}^1(\mathbf{W}_{:,S}) $-simplical
  \end{theorem} 
  \begin{proof}
  Let $S_2 = S \bigcap D^c$. Here, $ \lvert S_2 \rvert \geq 2m' (1-\frac{1}{16m}) \geq \frac{15m'}{8} \geq \frac{512}{21\tilde{c}} m \log (eK)$.
   The first result follows from the following chain: 
        \begin{align}
          \lVert \mathbf{a}^T ( \mathbf{W}_{:,S} )\rVert_1 & \geq  \lVert \mathbf{a}^T [\mathbf{W}_{:,S_1} \mathbf{W}_{:,S_2} ] \rVert_2 \\
          &\underset{(a)} \geq  \lVert \mathbf{a}^T \mathbf{W}_{:,S_2} \rVert_2 - \lVert \mathbf{a}^T \mathbf{W}_{:,S_1}  \rVert_2  \nonumber \\
            \hfill & \underset{(b)} \geq \lVert \mathbf{a} \rVert_2 \psi_m(\mathbf{W}_{S_2}) - \lVert \mathbf{a} \rVert_2 \sqrt{m \frac{2m'}{16m}} \\
            &\underset{(c)} \geq \lVert \mathbf{a} \rVert_1 \frac{\sqrt{15m'}}{\sqrt{8m}} \left( \frac{3}{4} -\frac{1}{5} -\frac{1}{\sqrt{15}} \right) \nonumber \\
            \hfill &\geq \left(\frac{3}{4} -\frac{8}{15} \right) \frac{\sqrt{15m'}}{\sqrt{8m}}  ~\mathrm{w.p.~} 1- \frac{2}{K^{7m/2}}
         \end{align}
       Justifications of the above chain are: (a)- Triangle inequality for the norm $\lVert \cdot \rVert_2$. (b)- Definition of $\psi_m(\cdot)$ and $ \lVert \mathbf{a}^T \mathbf{W}_{S_1}\rVert_2 \leq \lVert \mathbf{a}^T \rVert_2 \lvert \mathbf{W}_{S_1}\rVert_F \leq \sqrt{m \lvert S_1 \rvert} \lVert \mathbf{a}^T \rVert_2 $. (c)- $\lVert \cdot \rVert_2 \geq \frac{\lVert \cdot \rVert_1}{\sqrt{m}}$ and applying Theorem \ref{resrandW} because $S_2 \subseteq D^c$ and $\lvert S_2 \rvert \geq \frac{512}{21\tilde{c}} m \log (eK)$.
       
    For the second part, let us denote $\mathbf{r}^{-i} \in \mathbb{R}^{1 \times m}$ to be a vector satisfying $\sum \limits_{k \neq i} r^{-i}_k =-1$, $r^{-i}_k \leq 0 ~\forall k \neq i$ and $r^{-i}_i=1$. It is easy to easy that:
 \begin{equation} \label{eqn:geq1}
  \lVert \mathbf{r}^{-i} \rVert_1 \geq 1.
 \end{equation}
 From the definition for an $\alpha$-simplical matrix (Definition \ref{def:simplical}), it is enough to show that for any $\mathbf{r}^{-i}$, $\lVert \mathbf{r}^{-i} \mathbf{W}_S\rVert_1 \geq \psi_{m}^1 (\mathbf{W}_{:,S})$. We prove this as follows:
 \begin{align}
   \lVert \mathbf{r}^{-i} \mathbf{W}_{:,S} \rVert_1  \underset{(  \lVert \mathbf{r}^{-i} \rVert_1 \geq 1)}\geq \psi_{m}^1(\mathbf{W}_{:,S})
 \end{align}
  \end{proof}

\subsubsection{Choosing a good $S$ for $\mathbf{G}(0)$}
\begin{lemma}\label{Shigh}
   Let $D$ be the set as defined in Section \ref{sec:models}. Let a random $2m'$-subset $S$ be chosen out of $[K]$ where $m'= \frac{512}{21\tilde{c}} m \log (eK)$. Then, $ \mathrm{Pr} \left( \lvert S \bigcap D \rvert \leq \frac{2m'}{16m} \right) \leq \exp (- c_1 \log (eK))$ for constant $c_1>0$ that depends on $\tilde{c}$.
\end{lemma}
\begin{proof}
Let $X_1, \ldots X_{2m'}$ be set of indicator functions such that $X_i=1$ if the $i$-th element in the random subset $S$ chosen uniformly without replacement belongs to $D$ and it is $0$ otherwise. Let $Y_1,Y_2 \ldots Y_{2m'}$ be the set of indicator functions such that $Y_i=1$ (and $0$ otherwise) if the $i$-th element in the random multi-set $S$ belongs to $D$ where the multiset elements are chosen independently and uniformly with replacement. It is clear that $\mathbb{E}[X_i]=\mathbb{E}[Y_i]= \frac{\lvert D \rvert}{K} =\mu \geq \frac{1}{32m}$. The moment generating function of the sum of $X_i$'s is dominated by the moment generating function of the sum of $Y_i$'s. Therefore, all concentration inequalities, based on moment generating functions, for variables drawn with replacement holds for variables drawn without replacement \cite{hoeffding1963probability}. In particular, the following inequality derived from moment generating functions holds \cite{jukna2011extremal} for any $\delta>0$:
  \begin{align*}
    &\mathrm{Pr} \left( \sum X_i \geq (1+ \delta)2m'\mu \right) \\
    &\leq   \mathrm{Pr} \left( \sum X_i \geq (1+ \delta) 2m'\mu \right) \\
    &\leq  \exp \left( \delta 2m'\mu \right) \left( 1+\delta \right)^{- (1+\delta)2m'\mu}.
  \end{align*}
Let us take $\delta= 1$. Therefore, $ \mathrm{Pr} \left( \lvert S \bigcap D \rvert \geq \frac{2m'}{16m} \right) \leq  \left( \frac{4}{e} \right)^{-\frac{32}{21 \tilde{c}} \log (eK)}  \leq \frac{1}{(eK)^{\frac{32}{21 \tilde{c}} \log(4/e) } }$.

\end{proof}

\begin{proof}[Proof of Theorem~\ref{thm:goodW}]
From Theorem~\ref{proj1} and Lemma~\ref{Shigh} we have,
\begin{align*}
&\EE_{\mathbf{W}} \left[ \PP_{\mathbf{S}}\left(\psi^1_m(\mathbf{W}_{:,S}) < \left(\frac{13}{60} \right) \frac{\sqrt{15m'}}{\sqrt{8m}} \right) \right] \\
&\leq  \exp (- c_1 \log (eK)) + 2K^{-7m/2} \\
& \leq 2 \exp(- c_1 \log (eK)
\end{align*}
Now by Markov's inequality this implies that,
\begin{align*}
&\PP_{\mathbf{W}} \left[ \left(\PP_{\mathbf{S}}\left(\psi^1_m(\mathbf{W}_{:,S}) < \left(\frac{13}{60} \right) \frac{\sqrt{15m'}}{\sqrt{8m}} \right) \right. \right. 
 \left. \left. \geq  2\exp (- \frac{c_1}{2} \log (eK)) \right) \right]   \\
&\leq \frac{\exp(- c_1 \log (eK)}{\exp (- \frac{c_1}{2} \log (eK))} \\ 
& \leq \exp (- \frac{c_1}{2} \log (eK))
\end{align*}
This implies the following chain: 
\begin{align*}
&\PP_{\mathbf{W}} \left[ \left(\PP_{\mathbf{S}}\left(\psi^1_m(\mathbf{W}_{:,S}) > \left(\frac{13}{60} \right) \frac{\sqrt{15m'}}{\sqrt{8m}} \right)  \right. \right.  \left. \left. \leq 1- 2\exp (- \frac{c_1}{2} \log (eK)) \right) \right] \\
 &\leq \exp (- \frac{c_1}{2} \log (eK)) \\
&\Rightarrow \PP_{\mathbf{W}} \left[ \left(\PP_{\mathbf{S}}\left(\psi^1_m(\mathbf{W}_{:,S}) > \left(\frac{13}{60} \right) \frac{\sqrt{15m'}}{\sqrt{8m}} \right) \right. \right. \left. \left.\geq 1- 2\exp (- \frac{c_1}{2} \log (eK)) \right) \right] \\
& \geq 1- \exp (- \frac{c_1}{2} \log (eK)) 
\end{align*}
This proves that with probability at least $ 1 - \exp (- \frac{c_1}{2} \log (eK))$ the $\ell_1$-WStRIP condition is satisfied with the said parameters.
\end{proof}
 

\subsection{Results regarding sub-matrices of $\mathbf{A}$}

We assume that $\mathbf{A}$ satisfies the random generative model in \ref{sec:models}. We prove some results regarding the minimum singular values of sub-matrices corresponding to columns in set $S$ ($\lvert S \rvert = 2m'$) which is a mix of random and the deterministic columns. The proofs follow closely those of $\mathbf{W}$ in the previous section.

\begin{theorem}
\label{projA}
Let $\mathbf{A}$ follow the random generative model in Section \ref{sec:models}.  Let $m' \geq \frac{512 }{21\tilde{c}} m \log (eL)$. Fix any set $S$ of size $2m'$ such that $S_1 = S \bigcap E,~ \lvert S_1 \rvert \leq \frac{2m'}{9} $. Let $S_2= S \setminus S_1$. Then, we have: 
\begin{equation}
\sigma_{m} \left(\mathbf{A}_{S,:}\right) \geq \frac{\sqrt{m'}}{m} \left( \frac{1}{20} \right)  ~ \mathrm{w.p~} 1-\frac{2}{L^{7m/2}}.
\end{equation}

 \end{theorem}    
 \begin{proof}
Let $\tilde{S}_2$ be the set of rows in the random matrix $\mathbf{\tilde{A}}$ that corresponds to the rows $S_2$ in $\mathbf{A}$.
Here, $\mathbf{\hat{A}}_{\tilde{S}_2,:}$ has sub-gaussian entries with sub-gaussian parameter $c(q)$, since by Lemma \ref{subgauss}, all bounded random variables on support $[-1,1]$ with zero mean are sub-gaussian and their sub-gaussian parameter depends on the variance. 

Therefore, applying Lemma \ref{singlower} to the sub-gaussian matrix ($\lvert \tilde{S}_2 \rvert \times m$) $\mathbf{\hat{A}_{\tilde{S}_2,:}}$ with $ \lvert \tilde{S}_2 \rvert \geq m' \geq \frac{512 }{21\tilde{c}} m \log (eL) $ and setting $\delta =7/16$, we have:
  \begin{align*}
    \mathrm{Pr} \left( \sigma_m ( \mathbf{\hat{A}_{\tilde{S}_2,:}} ) \leq \frac{3}{4} \sqrt{m'} \right)  &\leq 2 \exp \left( - \frac{7 m \log (L) }{2}\right) \\
    &\leq 2 L^{-7m/2}.
      \end{align*}

Now, consider the following matrix: $[\frac{1}{m}\left( \mathbf{N}_{\tilde{S}_2,:}+\mathbf{\hat{A}}_{\tilde{S}_2,:}\right) \mathbf{A}_{S_1,:}]$. First, note that according to the model in Section \ref{sec:models}, rows of $\mathbf{A}_{S_1,:}$ sum to $1$. Therefore, we have the following chain for any non zero vector $\mathbf{a} \in \mathbb{R}^{1 \times m}$:
\begin{align}
   & \lVert [\mathbf{N}_{\tilde{S}_2,:}+\mathbf{\hat{A}}_{\tilde{S}_2,:} \mathbf{A}_{S_1,:}]  \mathbf{a} \rVert_2  \geq   \lVert \mathbf{N}_{\tilde{S}_2,:}+\frac{1}{m}\mathbf{\hat{A}}_{\tilde{S}_2,:}  \mathbf{a} \rVert_2 \\
   &- \lVert \mathbf{A}_{S_1,:} \mathbf{a} \rVert_2 \nonumber \\
     \hfill                                                                                                                    & \geq    \lVert \left( \mathbf{N}_{\tilde{S}_2,:}+\mathbf{\hat{A}}_{\tilde{S}_2,:} \right)  \mathbf{a} \rVert_2 - \sqrt{\sum \limits_{i \in S_1} \lVert \mathbf{A}_{i,:} \rVert_2^2 \lVert \mathbf{a} \rVert_2^2} \nonumber \\
      \hfill                                                                                                                  & \geq     \lVert \left( \mathbf{N}_{\tilde{S}_2,:}+\mathbf{\hat{A}}_{\tilde{S}_2,:} \right)  \mathbf{a} \rVert_2  - \sqrt{\sum \limits_{i \in S_1} \lVert \mathbf{A}_{i,:} \rVert_1^2 \lVert \mathbf{a} \rVert_2^2} \nonumber \\
      \hfill                                                                                                                 & \geq     \lVert  \left( \mathbf{N}_{\tilde{S}_2,:}+\mathbf{\hat{A}}_{\tilde{S}_2,:} \right)  \mathbf{a} \rVert_2  - \sqrt{2\rho m'}   \lVert \mathbf{a} \rVert_2   \nonumber \\
       \hfill                       & \geq  \lVert \mathbf{\hat{A}}_{\tilde{S}_2,:} \mathbf{a} \rVert_2  - \lVert  \mathbf{N}_{\tilde{S}_2,:} \mathbf{a} \rVert_2 - \sqrt{ 2\rho m'}  \lVert \mathbf{a} \rVert_2  \nonumber \\
        \hfill                    & \geq   \sigma_m \left( \mathbf{\hat{A}}_{\tilde{S}_2,:}   -  \sqrt{2m'(1-\frac{1}{9})}\frac{1}{5}  - \sqrt{ 2\frac{1}{9} m' }  \right)  \lVert \mathbf{a} \rVert_2 \nonumber \\
        \hfill                    & \geq \left( \frac{3}{4}- \frac{1}{5} - \frac{\sqrt{2\frac{1}{9}}}{\sqrt{1-\frac{1}{9}}}  \right) \sqrt{2m' (1-\frac{1}{9})} ~ \mathrm{w.p.~} 1-2L^{-7m/2}. \nonumber \\
        \hfill                     & \geq   \left( \frac{3}{4}- \frac{1}{5} - \frac{1}{2}  \right) \sqrt{m' }  ~ \mathrm{w.p.~} 1-2L^{-7m/2}. 
\end{align}

Now, we normalize the every row of $[\mathbf{N}_{\tilde{S}_2,:}+\mathbf{\hat{A}}_{\tilde{S}_2,:} \mathbf{A}_{S_1,:}] $ to get $[\mathbf{A}_{S_2,:} \mathbf{A}_{S_1,:}] = \mathbf{A}_{S}\mathbf{P}$ where $\mathbf{P}$ is a permutation matrix. Now, every entry gets scaled by at least $1/m$ since rows sum is at most $m$. Therefore, the minimum singular value scales by at least $1/m$. Therefore,
\begin{align*}
 \sigma_{m} \left(\mathbf{A}_{S}\right)= \sigma_{m} \left(\mathbf{A}_{S}\mathbf{P} \right) &\geq \frac{\sqrt{m'}}{m} \left( \frac{3}{4}- \frac{1}{5} - \frac{1}{2} \right) \\ 
 &\mathrm{w.p~} 1-2L^{-7m/2}.
\end{align*}

 \end{proof}

\subsubsection{Choosing a good $S(i)$ for a $\mathbf{G}(i)$ }

\begin{lemma}\label{Sigood}
   Let $E$ be the set as defined in Section \ref{sec:models}. Let a random $2m'$-subset $S$ be chosen out of $[L]$ where $m'= \frac{512}{21\tilde{c}} m \log (eL)$. Then, $ \mathrm{Pr} \left( \lvert S \bigcap E \rvert \leq \frac{2m'}{9} \right) \leq \exp (- c_2 m \log (eL))$ for constant $c_2>0$ that depends on $\tilde{c}$.
\end{lemma}
\begin{proof}
The proof is identical to the proof of Lemma \ref{Shigh}. We just choose $\mu=\frac{1}{18}$ and $\delta=1$. Therefore we have:
\begin{equation}
  \mathrm{Pr} \left( \lvert S \bigcap E \rvert \geq \frac{2m'}{9} \right) \leq \frac{1}{(eL)^{\frac{512 \log(4/e)}{189 \tilde{c}} m }}.
\end{equation}

\end{proof}

\begin{proof}[Proof of Theorem~\ref{thm:goodA}]
From Theorem~\ref{projA} and Lemma~\ref{Sigood} we have,
\begin{align*}
&\EE_{\mathbf{A}} \left[ \PP_{\mathbf{S}}\left(\sigma_m(\mathbf{A}_{S,:}) < \frac{\sqrt{m'}}{m} \left( \frac{1}{20} \right) \right)\right]\\ 
&\leq  3\exp (- c'_2m \log (eL))
\end{align*}
Now by Markov's inequality this implies that,
\begin{align*}
&\PP_{\mathbf{A}} \left[ \left( \PP_{\mathbf{S}}\left(\sigma_m(\mathbf{A}_{S,:}) < \frac{\sqrt{m'}}{m} \left( \frac{1}{20} \right) \right) \right. \right.  \left. \left. \geq  \exp (- \frac{c'_2}{2}m \log (eL)) \right)\right] \\
&\leq 3\frac{\exp (- c'_2m \log (eL))}{\exp (-\frac{ c'_2}{2}m \log (eL))} \\ 
& \leq 3\exp (- \frac{c'_2}{2}m \log (eL))
\end{align*}
This implies the following chain:
\begin{align*}
&\PP_{\mathbf{A}} \left[ \left( \PP_{\mathbf{S}}\left(\sigma_m(\mathbf{A}_{S,:}) > \frac{\sqrt{m'}}{m} \left( \frac{1}{20} \right) \right) \right. \right. \left. \left. \leq  1-\exp (- \frac{c'_2}{2}m \log (eL)) \right)\right] \\
& \leq 3\exp (- \frac{c'_2}{2}m \log (eL))  \\
&\Rightarrow \PP_{\mathbf{A}} \left[ \left( \PP_{\mathbf{S}}\left(\sigma_m(\mathbf{A}_{S,:}) > \frac{\sqrt{m'}}{m} \left( \frac{1}{20} \right) \right) \right. \right.  \left. \left. \geq  1-\exp (- \frac{c'_2}{2}m \log (eL)) \right)\right] \\
 &\geq 1-3\exp (- \frac{c'_2}{2}m \log (eL)) 
\end{align*}

This proves that with probability at least $ 1 - \exp (- \frac{c'_2}{2} m\log (eL))$ the $\ell_2$-WStRIP condition is satisfied with the said parameters.
\end{proof}

\subsection{Noisy NMF in Low dimensions}
In this section we enhance the guarantees of the robust Hottopix algorithm from ~\cite{gillis2014fast} provided  $\mathbf{W}$ satisfies $\ell_1$-WStRIP and the subset $S$ chosen by Algorithm~\ref{alg:LCB} is good as in Section~\ref{Shigh}. 
\begin{lemma}
\label{lem:nmf}
Suppose $\mathbf{W}$ satisfies $\ell_1$-WStRIP with parameter $(\delta,\rho_1,2m')$ and the subset $S$ of its columns ($\lvert S \rvert = 2m'$) satisfies $\psi^1_m(\mathbf{W}_{S,:}) \geq \rho_1$. Consider a matrix $\tilde{\mathbf{X}} = \mathbf{AW}_{:,S} + \mathbf{N}$ such that $\norm{\mathbf{N}}_{\infty,1} \leq \epsilon$ and $\mathbf{A}$ is separable~\cite{recht2012factoring}. Under these assumptions $\mathrm{Hottopix}(\tilde{\mathbf{X}},m,\epsilon)$ returns $\hat{\mathbf{W}}$ such that, 
\begin{equation}
\norm{\hat{\mathbf{W}} - \mathbf{W}_{:,S}}_{\infty,1} \leq \epsilon
\end{equation}
\label{eq:nmf1}
if $\epsilon < \frac{\rho_1(1 - \lambda)}{15}$. 
Suppose $\hat{\mathbf{A}} = \argmin_{\mathbf{Z} \geq \mathbf{0}, \mathrm{rowsum}(\mathbf{Z}) = \mathbf{1}} \norm{\tilde{\mathbf{X}} - \mathbf{Z}\hat{\mathbf{W}}}_{\infty,1}$. Then we have,
\begin{equation}
\label{eq:nmf2}
\norm{\hat{\mathbf{A}} -\mathbf{A}}_{\infty,1} \leq \frac{4\epsilon}{\rho_1 - \epsilon} 
\end{equation} 
\end{lemma} 

\begin{proof}
Let $\mathbf{W'} = \mathbf{W}_{:,S}$ and $\mathbf{X} = \mathbf{AW}_{:,S}$. The bound in \eqref{eq:nmf1} is immediate from Theorem 2 in~\cite{gillis2014fast} as $\mathbf{W'}$ is $\rho_1$-robust simplical by Theorem~\ref{proj1}. We first note that,
\begin{align*}
\norm{\tilde{\mathbf{X}} - \mathbf{A}\hat{\mathbf{W}}}_{\infty,1} &\leq \norm{\tilde{\mathbf{X}} - \mathbf{X}}_{\infty,1} +  \norm{X - \mathbf{A}\mathbf{W'}}_{\infty,1} + \norm{ \mathbf{A}\mathbf{W'}- \mathbf{A}\hat{\mathbf{W}}}_{\infty,1} \\
& \leq \norm{\mathbf{A}\left(\mathbf{W'}- \hat{\mathbf{W}}\right)}_{\infty,1} + \epsilon\\
& \leq \norm{\mathbf{A}}_{\infty,1}\norm{\mathbf{W'}- \hat{\mathbf{W}}}_{\infty,1} + \epsilon \leq 2\epsilon
\end{align*}
The first inequality follows from the triangle inequality while the last one holds because $\norm{\mathbf{A}}_{\infty,1} = 1$. Thus, the LP to recover $\hat{\mathbf{A}}$ will always output $\hat{\mathbf{A}}$ with,
\begin{equation}
\label{eq:pbound}
\norm{\mathbf{X} - \mathbf{\mathbf{A}}\hat{\mathbf{W}}}_{\infty,1} = \norm{\mathbf{A}\mathbf{W'} - \hat{\mathbf{A}}\hat{\mathbf{W}}}_{\infty,1} \leq 3\epsilon. 
\end{equation} 
We can apply triangle inequality to get,
\begin{align*}
\norm{\left(\mathbf{A} - \hat{\mathbf{A}}\right)\mathbf{W'}}_{\infty,1} &\leq \norm{\mathbf{A}\mathbf{W'} - \hat{\mathbf{A}}\hat{\mathbf{W}}}_{\infty,1} + \norm{\mathbf{\hat{A}} \left(\mathbf{W'} - \hat{\mathbf{W}}\right)}_{\infty,1} \\
& \leq 3\epsilon +  \norm{\mathbf{\hat{A}}}_{\infty,1} \norm{\mathbf{W'} - \hat{\mathbf{W}}}_{\infty,1}\\
& \leq 3\epsilon +  \left(1 + \norm{\mathbf{\hat{A}} - \mathbf{A}}_{\infty,1} \right)\epsilon \numberthis \label{eq:rhs}
\end{align*}
In order to get the desired result we need to lower bound the L.H.S in \eqref{eq:rhs}. Note that $\mathrm{rowsum}\left(\mathbf{A} - \hat{\mathbf{A}}\right) = \mathbf{0}$. Therefore we have, 
\begin{equation}
\norm{\left(\mathbf{A} - \hat{\mathbf{A}}\right)\mathbf{W'}}_{\infty,1} \geq \norm{\mathbf{A} - \hat{\mathbf{A}}}_{\infty,1}\rho_1 
\label{eq:lhs}
\end{equation}
by definition. 
Combining \eqref{eq:lhs} and \eqref{eq:rhs} we get the required bound.  
\end{proof}

\subsection{Noisy Recovery of Extreme Points}
In this section we assume that $\mathbf{A}$ satisfies the $\ell_2$-WStRIP property with parameter $(\delta / L,\rho_2,m')$. 
\begin{lemma}
\label{lem:union}
If $\mathbf{A}$ satisfies the $\ell_2$-WStRIP property with parameter $(\delta / L,\rho_2,2m')$ then the sets $\{S(1),\cdots, S(l+1)\}$ with $\lvert S(i)\rvert = 2m'$ satisfy, 
\begin{align*}
\sigma_{m}(\mathbf{A}_{S(i),:}) \geq \rho_2, \text{ for all } i \in [l+1] 
\end{align*}
with probability atleast $1 - \delta$ over the randomness in choosing the subsets. 
\end{lemma}
\begin{proof}
The proof of this lemma is just an union bound over all the events $\left\{\sigma_{m}(\mathbf{A}_{S(i),:}) < \rho_2 \right\}$. Note that by virtue of $\ell_2$-WStRIP each of these events is true with probability atmost $\delta/L$. 
\end{proof}

If the conditions of the above lemma are satisfied we will call the corresponding sets \textit{good}. Recall the definition of $\hat{\mathbf{M}}_i(t)$. We will show that if $\hat{\mathbf{A}}(t)$ is close to $\mathbf{A}$ and the matrices $\hat{\mathbf{M}}_i(t)$ are sufficiently close to their means, then we recover $\mathbf{W}$ upto the same accuracy. 
Let us define $\mathbf{M}_i = \EE \left[ \hat{\mathbf{M}}_i(t) \right]$. 
\begin{lemma}
\label{lem:erecov}
Suppose $\mathbf{A}$ satisfies the $\ell_2$-WStRIP property and $\{S(1), S(2), \cdots S(l+1) \}$ are good in the sense of Lemma~\ref{lem:union}. Given that $\norm{\hat{\mathbf{A}}(t) - \mathbf{A}}_{\infty,1} \leq \epsilon_1$ and $\norm{\hat{\mathbf{M}}_i(t)-\mathbf{M}_i}_{\infty,\infty} \leq \epsilon_2$ for all $i \in [l+1]$,  $\hat{\mathbf{W}}(t)$ recovered by Algorithm~\ref{alg:LCB} satisfies,
\begin{equation}
\norm{\hat{\mathbf{W}}(t) - \mathbf{W}}_{\infty,\infty} \leq \frac{m(2\epsilon_1 + 3\epsilon_2)}{\rho_2}
\end{equation}  
if $\epsilon_1, \epsilon_2 \leq \frac{\rho_2}{m}$. 
\end{lemma}
\begin{proof}
Let $\hat{\mathbf{W}}(t)_{:,(i-1)m:im - 1}$ and $\mathbf{W}_{:,(i-1)m:im - 1}$ be denoted by $\hat{\mathbf{W}}_i(t)$ and $\mathbf{W}_i$ respectively. Similarly we denote $\hat{\mathbf{A}}(t)_{S(i),:}$ and $\mathbf{A}_{S(i),:}$ by $\hat{\mathbf{A}}_i(t)$ and $\mathbf{A}_i$ respectively. Then following identities hold,
\begin{align*}
\mathbf{A}_i\mathbf{W}_i &= \mathbf{M}_i \\
\hat{\mathbf{A}}_i(t)\hat{\mathbf{W}}_i(t) &= \hat{\mathbf{M}}_i(t) \numberthis \label{eq:ident}
\end{align*} 
Note that $ \mathbf{A}_i$ has full-column rank. Let the left-inverse of $ \mathbf{A}_i$ be $\mathbf{A}^*_i$. It is easy to see that,
\begin{equation}
\norm{\mathbf{A}^*_i}_{\infty,1} \leq \frac{m}{\rho_2}. \label{eq:inverse} 
\end{equation}
From~\eqref{eq:ident} we have,
\begin{align*}
&\left(\mathbf{I} + \mathbf{A}^*_i(\hat{\mathbf{A}}_i(t) - \mathbf{A}_i) \right) \hat{\mathbf{W}}_i(t)  = \mathbf{W}_i + \mathbf{A}^*_i(\hat{\mathbf{M}}_i(t) - \mathbf{M}_i) \\
&\implies \hat{\mathbf{W}}_i(t)  = \left(\mathbf{I} + \mathbf{A}^*_i(\hat{\mathbf{A}}_i(t) - \mathbf{A}_i) \right)^{-1}\left(\mathbf{W}_i + \mathbf{A}^*_i(\hat{\mathbf{M}}_i(t) - \mathbf{M}_i)\right) \\
&\implies \hat{\mathbf{W}}_i(t)  = \left(\mathbf{I} - \mathbf{A}^*_i(\hat{\mathbf{A}}_i(t) - \mathbf{A}_i)(\mathbf{I} + \mathbf{A}^*_i(\hat{\mathbf{A}}_i(t) - \mathbf{A}_i)) \right)(\mathbf{W}_i 
+ \mathbf{A}^*_i(\hat{\mathbf{M}}_i(t) - \mathbf{M}_i) )
\end{align*}
We can simplify further to yield,
\begin{align*}
&\hat{\mathbf{W}}_i(t) - \mathbf{W}_i = \mathbf{A}^*_i(\hat{\mathbf{M}}_i(t) - \mathbf{M}_i) - \left(\mathbf{A}^*_i(\hat{\mathbf{A}}_i(t) - \mathbf{A}_i)\mathbf{W}_i + \left( \mathbf{A}^*_i(\hat{\mathbf{A}}_i(t) - \mathbf{A}_i)\right)^{2}\mathbf{W}_i\right) \\
&- \left(\mathbf{A}^*_i(\hat{\mathbf{A}}_i(t) - \mathbf{A}_i)\mathbf{A}^*_i(\hat{\mathbf{M}}_i(t) - \mathbf{M}_i) \right.  \left. + \left( \mathbf{A}^*_i(\hat{\mathbf{A}}_i(t) - \mathbf{A}_i)\right)^{2}\mathbf{A}^*_i(\hat{\mathbf{M}}_i(t) - \mathbf{M}_i) \right)
\end{align*}
Therefore by triangle inequality we have,
\begin{align*}
&\norm{\hat{\mathbf{W}}_i(t) - \mathbf{W}_i}_{\infty,1}  = \norm{\mathbf{A}^*_i(\hat{\mathbf{M}}_i(t) - \mathbf{M}_i)}_{\infty,1} \\
&+ \norm{\left(\mathbf{A}^*_i(\hat{\mathbf{A}}_i(t) - \mathbf{A}_i)\mathbf{W}_i + \left( \mathbf{A}^*_i(\hat{\mathbf{A}}_i(t) - \mathbf{A}_i)\right)^{2}\mathbf{W}_i\right)}_{\infty,1} \\
&+ \norm{\left(\mathbf{A}^*_i(\hat{\mathbf{A}}_i(t) - \mathbf{A}_i)\mathbf{A}^*_i(\hat{\mathbf{M}}_i(t) - \mathbf{M}_i) \right.\right.  \left.\left.+ \left( \mathbf{A}^*_i(\hat{\mathbf{A}}_i(t) - \mathbf{A}_i)\right)^{2}\mathbf{A}^*_i(\hat{\mathbf{M}}_i(t) - \mathbf{M}_i) \right)}_{\infty,1}
\end{align*}
Now we will bound each of the terms seperately as follows,
\begin{align*}
 \norm{\mathbf{A}^*_i(\hat{\mathbf{M}}_i(t) - \mathbf{M}_i)}_{\infty,1} &\leq\norm{\mathbf{A}^*_i}_{\infty,1}\norm{(\hat{\mathbf{M}}_i(t) - \mathbf{M}_i)}_{\infty} \\
 & \leq  \frac{m\epsilon_2}{\rho_2}
\end{align*}
Similarly we have,
\begin{align*}
&\norm{\left(\mathbf{A}^*_i(\hat{\mathbf{A}}_i(t) - \mathbf{A}_i)\mathbf{W}_i + \left( \mathbf{A}^*_i(\hat{\mathbf{A}}_i(t) - \mathbf{A}_i)\right)^{2}\mathbf{W}_i\right)}_{\infty,1} \\
& \leq \norm{\mathbf{A}^*_i}_{\infty,1}(1 + \norm{\mathbf{A}^*_i}_{\infty,1}\epsilon_1)\epsilon_1\norm{\mathbf{W}_i}_{\infty,\infty} \\
& \leq  \frac{2m\epsilon_1}{\rho_2}
\end{align*}
Finally the third term can be bounded as, 
\begin{align*}
& \norm{\left(\mathbf{A}^*_i(\hat{\mathbf{A}}_i(t) - \mathbf{A}_i)\mathbf{A}^*_i(\hat{\mathbf{M}}_i(t) - \mathbf{M}_i) \right. \right. \left. \left. + \left( \mathbf{A}^*_i(\hat{\mathbf{A}}_i(t) - \mathbf{A}_i)\right)^{2}\mathbf{A}^*_i(\hat{\mathbf{M}}_i(t) - \mathbf{M}_i) \right)}_{\infty,1} \\
& \leq \left(\norm{\mathbf{A}^*_i}_{\infty,1}\right)^2\epsilon_1\epsilon_2 + \left( \norm{\mathbf{A}^*_i}_{\infty,1}\right)^3\epsilon_1^2\epsilon_2 \leq \frac{2m\epsilon_2}{\rho_2}
\end{align*}
Therefore we have,
\begin{align*}
\norm{\hat{\mathbf{W}}_i(t) - \mathbf{W}_i}_{\infty,1} \leq \frac{m(2\epsilon_1 + 3\epsilon_2)}{\rho_2}
\end{align*}
We can repeat the same analysis for all $i \in [l+1]$ to arrive at the required result.
\end{proof}

\subsection{Putting it together: Online Analysis}
\label{sec:online}
In this section we prove Theorem~\ref{thm:mainub}, which provides a parameter dependent upper bound on the regret of Algorithm~\ref{alg:LCB} if $\mathbf{W}$ and $\mathbf{A}$ satisfy the $\ell_1$-WStRIP and $\ell_2$-WStRIP. The regret bound provided here is in the parameter dependent regime, that is we assume a constant gap between the best arm and the rest for each context. More precisely let $\Delta = \min_{s \in [L]} \left( u^*(s) - \max_{k \neq k*(s)} U_{sk} \right)$ be a fixed constant not scaling with $L,K$ or $t$. This falls under the purview of the random generative model because we allow for $\Theta(K/m)$ deterministic rewards for each of the latent context. These conditions are expected to hold in real world data as each latent contexts are expected to have some unique arms which are significantly different from the others. In the said regime we reduce the regret bound of $O\left(LK\log(t)\right)$ for general contextual bandit to only an $O \left( L\mathrm{poly}(m,\log(K))\log(T) \right)$ dependence. 
\begin{theorem}
\label{thm:mainub}
In a contextual bandit setting suppose the reward matrix has the form $\mathbf{U} = \mathbf{A}\mathbf{W}$ and each contexts $s$ arrives independently with probability $\beta_s$ for all $s \in [L]$. Assume that $L = \Omega(K\log(K))$. If the problem parameters satisfy the following assumptions,
\begin{itemize}
\item $\beta = \min_s \beta_s = \Omega\left( 1/L\right)$.
\item $\mathbf{W} \in \mathbb{R}^{m \times K}$ satisfies $\ell_1$-WStRIP with parameters $\left(\delta,\rho_1,2m' \right)$
\item $\mathbf{A} \in [0,1]^{L \times m}$ satisfies $\ell_2$-WStRIP with parameters $\left(\delta/L,\rho_2,2m' \right)$ and is separable~\cite{recht2012factoring}.  
\end{itemize}
then with probability atleast $1 - \delta$, Algorithm~\ref{alg:LCB} with $\epsilon_t = \min \left(1,\frac{\theta(2m'+m)}{\beta t}\right)$ and $\gamma(t) = \max \left(\frac{1}{t}, \frac{2}{\sqrt{\theta}} \right)$ has regret,
\begin{align*}
R(T) &  \leq \frac{\theta(m+2m')\log(T)}{\beta} + 4(L+K+1)m' \log(T) + o(1) \\
& = O \left( L\frac{\mathrm{poly}(m,m')}{\Delta^2}\log T  \right) \\
& = O \left(L \frac{m^5\log ^2K}{\Delta^2} \log T \right)
\end{align*}
where $\theta \geq 4\max \left( \frac{2m'((16 + \Delta)\rho_2 + 32m) }{\Delta\rho_1\rho_2}, \frac{15}{\rho_1(1 - \lambda)} \right)^2$.
\end{theorem}

Before we proceed to the proof of our theorem, we need to introduce a few useful lemmas. The next lemma connects the chance of making an error in the \textit{exploit} phase with the estimation errors in the system. 

\begin{lemma}
\label{lem:online}
Suppose at time $t$, $\norm{\hat{\mathbf{F}}(t) - \mathbf{F}}_{\infty,\infty} \leq \epsilon_1(t)$ and $\norm{\hat{\mathbf{M}}_i(t) - \mathbf{M}_i}_{\infty} \leq \epsilon_2(t)$ for all $i \in [l+1]$. If the following conditions hold,
\begin{align*}
\epsilon_1(t) &\leq \min \left( \frac{\Delta\rho_1\rho_2}{2m'((16 + \Delta)\rho_2 + 32m) }, \frac{\rho_1(1 - \lambda)}{15} \right)\\
\epsilon_2(t) &\leq \frac{\Delta\rho_2}{12m}\\           \numberthis \label{eq:epdelta}
E(t) &= 0
\end{align*}
then $k(t) = k^*(s_t)$, that is the optimal arm for the context is scheduled in the \textit{exploit} phase.

\end{lemma}

\begin{proof}
If $\epsilon_1(t) \leq \frac{\rho_1(1 - \lambda)}{15}$, then by Lemma~\ref{lem:nmf} we have,
\begin{equation}
\norm{\hat{\mathbf{A}}(t) -\mathbf{A}}_{\infty,1} \leq \frac{8m'\epsilon_1(t)}{\rho_1 - 2m'\epsilon_1(t)} 
\label{eq:abound}
\end{equation}
Since we have,
\begin{align*}
\epsilon_1(t) &\leq \frac{m\rho_1}{2m'(4\rho_2 + m)} \\
\epsilon_2(t) &\leq \frac{\rho_2}{m} 
\end{align*}
it is easy to verify that the conditions of Lemma~\ref{lem:erecov} are satisfied. Therefore we have,
\begin{equation}
\label{eq:erecov}
\norm{\hat{\mathbf{W}}(t) - \mathbf{W}(t)}_{\infty,\infty} \leq \frac{m}{\rho_2}\left(\frac{16m'\epsilon_1(t)}{\rho_1 - 2m'\epsilon_1(t)}  + 3\epsilon_2(t) \right)
\end{equation}
Therefore we have,
\begin{align*}
&\norm{\hat{\mathbf{U}}(t) - \mathbf{U}}_{\infty,\infty} = \norm{\mathbf{A}\mathbf{W} - \hat{\mathbf{A}}(t)\hat{\mathbf{W}}(t)}_{\infty,\infty} \\
& \leq \norm{\mathbf{A}}_{\infty,1}\norm{\mathbf{W} - \hat{\mathbf{W}}(t) }_{\infty,\infty} + \norm{\mathbf{A} - \hat{\mathbf{A}}(t)}_{\infty,1}\norm{\hat{\mathbf{W}}(t) }_{\infty,\infty}\\
& \leq \frac{m}{\rho_2}\left(\frac{16m'\epsilon_1(t)}{\rho_1 - 2m'\epsilon_1(t)}  + 3\epsilon_2(t) \right) +  \frac{8m'\epsilon_1(t)}{\rho_1 - 2m'\epsilon_1(t)} \\
& \leq \frac{8m'\epsilon_1(t)}{\rho_1 - 2m'\epsilon_1(t)} \left(1 + \frac{2m}{\rho_2} \right) + 3\frac{m\epsilon_2(t)}{\rho_2}
\end{align*}
Now, under the conditions of the lemma in \eqref{eq:epdelta}, we have
\begin{align*}
\frac{8m'\epsilon_1(t)}{\rho_1 - 2m'\epsilon_1(t)} \left(1 + \frac{2m}{\rho_2} \right)  & \leq \frac{\Delta}{4}\\
3\frac{m\epsilon_2(t)}{\rho_2} &\leq \frac{\Delta}{4}
\end{align*}
This further implies that,
\begin{align*}
\norm{\hat{\mathbf{U}}(t) - \mathbf{U}}_{\infty,\infty} \leq \frac{\Delta}{2}
\end{align*}
This guarantees that we select the optimal arm at time-step $t$. 
\end{proof}
The following lemma we prove that each entry of the matrices $\hat{\mathbf{F}}(t)$ and $\hat{\mathbf{M}}_i(t)$ 
for all $i \in [l+1]$ are sampled sufficient number of times. Let $T_{sj}(t)$ denote the the number of samples obtained for the entry $\hat{\mathbf{F}}(t)_{sj}$. Similarly we define $N^{(i)}(t)_{sj}$ as the number of sampled for the enrty $\hat{\mathbf{M}}_i(t)_{sj}$.    
\begin{lemma}
\label{lem:samples1}
Suppose $\epsilon_t = \frac{(m+2m')\theta}{\beta t}$ where $\beta = \min_{s} \beta_s$. Algorithm~\ref{alg:LCB} ensures that,
\begin{align*}
\PP \left( T_{sj}(t)  < \frac{\theta}{2} H_t\right) &\leq \frac{1}{t^{\theta/12}} \\
\PP \left( N^{(i)}(t)_{sj}  < \frac{\theta}{2} H_t\right) &\leq \frac{1}{t^{\theta/12}}
\end{align*}
 and where $H_n = \sum_{i = 1}^{n} \frac{1}{i} \sim \log(n)$ 
\end{lemma}

\begin{proof}
Let $S_t$ denote the random variable describing the context at time $t$. Let $C_t$ denote the random variable denoting the the column of $\mathbf{G}(0)$ to be sampled provide $E(t) = 1$ and $H_t = 1$. 
Note that,
\begin{align*}
\EE \left[T_{sj}(t)\right] & \geq  \sum_{l = 1}^{t} \PP \left( S_l = s, E(l) = 1 , H_l = 1, C _l = j  \right) \\
& \geq \sum_{l = 1}^{t} \frac{\theta}{l} = \theta H_{t}
\end{align*}
Now, a straight forward application of Chernoff-Hoeffding's inequality yields, 
 \begin{align*}
 \PP \left( T_{sj}(t)  <  (1 - \delta)\EE \left[T_{sj}(t)\right] \right) & \leq \exp \left(-\frac{\delta^2}{3} \EE \left[T_{sj}(t)\right]  \right) \\
 & \leq \exp \left(-\frac{\delta^2}{3} \theta H_{t}  \right) 
 \end{align*}
We can set $\delta = 1/2$ to get the required result. The same analysis works for $N^{(i)}(t)_{sj}$. The corresponding entry is sampled if $S_t = s_s(i)$. Let $C'_t$ denote the column of $\mathbf{G}(i)$ to be sampled when $E(t) = 1$,$S_t = s_s(i)$ and $H_t = 0$.
\begin{align*}
\EE \left[N^{(i)}(t)_{sj}\right] & \geq  \sum_{l = 1}^{t} \PP \left( E(t) = 1,S_t = s_s(i),H_t = 0, C' _l = j  \right) \\
& \geq \sum_{l = 1}^{t} \frac{\theta}{l} = \theta H_{t}
\end{align*}
\end{proof}
The same concentration inequality as before applies. 
\begin{lemma}
\label{lem:Finf}
Under the conditions of Lemma~\ref{lem:samples1} we have,
\begin{align*}
&\PP \left(\norm{\hat{\mathbf{F}}(t) - \mathbf{F} }_{\infty,\infty} > \epsilon_1(t) \right) \leq 4Lm'\exp\left(-\frac{\epsilon_1(t)^2}{2}\frac{\theta\log(t)}{2}\right)  +  \frac{2Lm'}{t^{\theta/12}}
\end{align*}
\end{lemma}
\begin{proof}
The proof of this lemma is an application of Chernoff's bound to the samples observed. Note that $\EE \left[ \hat{\mathbf{F}}(t)\right] = \mathbf{F} $. We have,
\begin{align*}
&\PP \left(\lvert \hat{\mathbf{F}}(t)_{sj} - \mathbf{F}_{sj} \rvert > \epsilon_1(t) \right)  \leq \PP \left(\lvert \hat{\mathbf{F}}(t)_{sj} - \mathbf{F}_{sj} \rvert > \epsilon_1(t) \bigg \vert T_{sj}(t)  \geq \frac{\theta}{2} H_t\right) + \PP \left( T_{sj}(t)  < \frac{\theta}{2} H_t \right) \\
& \leq 2e^{-\frac{\epsilon_1(t)^2}{2}\frac{\theta\log(t)}{2}} +  \frac{1}{t^{\theta/12}}
\end{align*}
where the last inequality if due to lemma~\ref{lem:samples1}. 
Now, we can apply an union bound over all $s \in [L]$ and $j \in [m]'$ to obtain the required result. 
\end{proof}
Similarly we can bound the errors in estimating $\mathbf{M}_i$'s as in the lemma below. 
\begin{lemma}
\label{lem:minf}
Under the conditions of Lemma~\ref{lem:samples1} we have, 
\begin{align*}
&\PP \left(\cup_{i \in [l+1]} \left\{\norm{\hat{\mathbf{M}}_i(t) - \mathbf{M}_i }_{\infty,\infty} > \epsilon_2(t) \right\}\right) \leq 4(K+1)m'\exp\left(-\frac{\epsilon_2(t)^2}{2}\frac{\theta\log(t)}{2}\right) +  \frac{2(K+1)m'}{t^{\theta/12}}
\end{align*}
\end{lemma}
\begin{proof}
The proof of this lemma is analogous to that of Lemma~\ref{lem:Finf}. We have the following chain, 
\begin{align*}
&\PP \left(\lvert \hat{\mathbf{M}}_{i}(t)_{sj} - \mathbf{M_i}_{sj} \rvert > \epsilon_1(t) \right) \leq \PP \left(\lvert \hat{\mathbf{M}}_{i}(t)_{sj} - \mathbf{M_i}_{sj}  \rvert > \epsilon_1(t) \bigg \vert T_{sj}(t)  \geq \frac{\theta}{2} H_t\right) + \PP \left( T_{sj}(t)  < \frac{\theta}{2} H_t \right) \\
& \leq 2e^{-\frac{\epsilon_2(t)^2}{2}\frac{\theta\log(t)}{2}} +  \frac{1}{t^{\theta/12}}
\end{align*}
We can apply union bound over all the entries of all the $l+1$ matrices to get the result. 
\end{proof}
Now, we are at a position to prove our main theorem. 
\begin{proof}[Proof of Theorem~\ref{thm:mainub}]
We have $\epsilon _t =  \frac{(m+2m')\theta}{\beta t}$ where we set,
\begin{equation}
\theta \geq 4\max \left( \frac{2m'((16 + \Delta)\rho_2 + 32m) }{\Delta\rho_1\rho_2}, \frac{15}{\rho_1(1 - \lambda)} \right)^2
\end{equation}
By virtue of the $\ell _1$-WStRIP property of $\mathbf{W}$, the set $S$ is $\rho_1$-simplical with probability at least $1 - \delta$. Similarly, by Lemma~\ref{lem:union} all the sets $S(i)$ are \textit{good} with probability at least $1 - \delta$. In what follows, we will assume that the above high probability conditions hold.
Note that according to Lemmas \ref{lem:Finf} and \ref{lem:minf} we have, 
\begin{align*}
&\PP \left(\norm{\hat{\mathbf{F}}(t) - \mathbf{F} }_{\infty,\infty} > \frac{2}{\sqrt{\theta}} \right) \leq \frac{4Lm'}{t} + o\left(\frac{1}{t^{2}}\right) \\
&\PP \left(\cup_{i \in [l+1]} \left\{\norm{\hat{\mathbf{M}}_i(t) - \mathbf{M}_i }_{\infty,\infty} > \frac{2}{\sqrt{\theta}} \right\}\right) \leq \frac{4(K+1)m'}{t} + o\left(\frac{1}{t^{2}}\right) \numberthis \label{eq:errors}
\end{align*}
As $\mathbf{U} \in [0,1]^{L \times K}$ the regret till time $T$ can be bounded as follows, 
\begin{align*}
&R(T) \leq \sum_{t=1}^{T} \EE \left[\mathds{1} \left\{E(t) = 1 \right\}\right] + \sum_{t=1}^{T} \EE \left[ \mathds{1} \left\{E(t) = 0 \right\} \right]\PP \left(k(t) \neq k^*(s_t) \right) \numberthis \label{eq:regretbound}
\end{align*}
By Lemma~\ref{lem:online} we have that,
\begin{align*}
&\PP \left( k(t) \neq k^*(s_t) \right)  \leq \PP \left(\norm{\hat{\mathbf{F}}(t) - \mathbf{F} }_{\infty,\infty} > \frac{2}{\sqrt{\theta}} \right) + \PP \left(\cup_{i \in [l+1]} \left\{\norm{\hat{\mathbf{M}}_i(t) - \mathbf{M}_i }_{\infty,\infty} > \frac{2}{\sqrt{\theta}} \right\}\right)
\end{align*}
We can combine this with \eqref{eq:regretbound} to get,
\begin{align*}
R(T) & \leq \frac{\theta(m+2m')\log(T)}{\beta} + 4(L+K+1)m' \log(T) + o(1) \\
& = O \left( L\mathrm{poly}(m,m')\log(T) \right)
\end{align*}
if we assume that $1/\beta = O(L)$. 
\end{proof}

\subsection{Lower Bound for $\alpha$-consistent Policies}
\label{sec:lboundproof}

In this section we provide a problem dependent lower bound for the contextual bandit problem with \textit{latent} contexts. The lower bound is established for a particular class of data-matrix $\mathbf{U}$ and for  $\alpha$-consistent policies. For, any $z_i \in \mathcal{Z}$ we define $\mathcal{C}(z_i)$ as, 
\begin{align*}
\mathcal{C}(z_i) := \left\{s \in \mathcal{S} : \alpha_{si} \neq 0 \right\} 
\end{align*}

\begin{theorem}
\label{thm:lb}
Consider a problem instance $(\mathbf{U},\mathbf{A},\mathbf{W})$ such that $\beta _s = 1/L$ for all $s \in \mathcal{S}$ and $\lvert \mathcal{C}(z_i) \rvert = L/m$ (assume that $m$ divides $L$) for all $z_i \in \mathcal{Z}$. Further, we assume that $\mathcal{C}(z_i)  \cap \mathcal{C}(z_j) = \emptyset$, for all $z_i \neq z_j$. Then the regret of any $\alpha$-consistent policy is lower-bounded as follows,
\begin{align*}
&R(T) \geq (K-1)mD(\mathbf{U})\left((1-\alpha)(\log(T/2m) - \log(L/m)) \right. \left.- \log(4KC) \right)
\end{align*}
for any $T > \tau$, where $C,\tau$ are universal constants independent of problem parameters and $D(\mathbf{U})$ is a constant that depends on the entries of $\mathbf{U}$ and is independent of $L,K$ and $m$. 
\end{theorem}

In order to prove Theorem~\ref{thm:lb} we introduce an inequality from the hypothesis testing literature. 

\begin{lemma}[\cite{tsybakov2008introduction}]
\label{lem:tsybakov}
Consider two probability measures $P$ and $Q$, both absolutely continuous with respect to a given measure. Then for any event $\mathcal{A}$ we have:
\begin{align*}
P(\mathcal{A}) + Q(\mathcal{A}^c) \geq \frac{1}{2} \exp \{ -\min (\mathrm{KL}(P||Q),\mathrm{KL}(Q||P)) \}
\end{align*}

\end{lemma}

\begin{proof}[Proof of Theorem~\ref{thm:lb}]
Note that the conditions in the theorem imply that there are $m$ distinct \textit{latent} contexts and there are $L/m -1$ copies for each of them. For any $z_i \in \mathcal{Z}$ let us define $T(z_i) = \sum_{t=1}^T \mathds{1} \left\{ S_t \in \mathcal{C}(z_i)\right\}$. With some abuse of notation we also define $k^*(z_i)$ as the index of the optimal arm and $\Delta(z_i)$ as the gap between the optimal and second optimal arm for all contexts in $\mathcal{C}(z_i)$. By the assumptions in the theorem we have, 
\begin{align*}
\mathbb{E}\left[ T(z_i)\right] = \frac{T}{m}
\end{align*}
Let $E_i$ be the event $\left\{ \frac{T}{2m} \leq T(z_i) \leq \frac{2T}{m}\right\}$. Let $E^c= \left\{ \cup _{z_i \in \mathcal{Z}} E_i ^c\right\}$. By a simple application of Chernoff bound we have,
\begin{align*}
\mathbb{P} \left( \cup _{z_i \in \mathcal{Z}} E_i ^c\right) \leq 2me^{-T/12} = o \left(\frac{1}{T^2} \right)
\end{align*}
Fix a $z_i \in \mathcal{Z}$ and let $k$ be the index of an arm that is not optimal for any of the contexts that belong to $\mathcal{C}(z_i)$. Let us create another system with parameter $(\mathbf{U'},\pmb{A},\mathbf{W}')$ where we make the entry $W_{ik} = \lambda = \frac{U_{max} + 1}{2}$ where $U_{max} = \max_{s,k}U_{sk}$, while everything else remains the same including the coefficients of the convex combinations relating the observed contexts to the \textit{latent} contexts. Note that this implies that in the second system arm $k$ is optimal for all $s \in \mathcal{C}(z_i)$. Let $A$ be the event defined as follows, 
\begin{align*}
A := \left\{\sum_{\{t : S_t \in \mathcal{C}(z_i)\}} \mathds{1} \left\{ X_t = k\right\} \geq \frac{T(z_i)}{2} \right\}
\end{align*}
Now, in the system with parameter $\mathbf{U}$ for any $s \in \mathcal{C}(z_i)$ we have, 
\begin{align*}
\mathbb{E}\left[ \sum_{\{t : S_t = s\}} \mathds{1} \left\{ X_t = k\right\}\right] \leq C T(s)^{\alpha}
\end{align*}
if $T(s) \geq \tau$, since the policy in consideration is $\alpha$-consistent. Here, $\tau, C$ are universal constants. 
By an application of Jensen's inequality we have,
\begin{align*}
\mathbb{E}\left[ \sum_{\{t : S_t \in \mathcal{C}(z_i)\}} \mathds{1} \left\{ X_t = k\right\}\right] \leq C \lvert\mathcal{C}(z_i)\rvert^{1-\alpha} T(z_i)^{\alpha}
\end{align*}
 Let $\mathbb{P} ^T _{\mathbf{U}}$ and $\mathbb{P} ^T _{\mathbf{U'}}$ be the distributions corresponding to the chosen arms and rewards obtained for $T$ plays for the two instances under a fixed $\alpha$-consistent policy. Now we can apply Markov's inequality to conclude that, 
\begin{align*}
 \mathbb{P}_{\mathbf{U}}(A) &\leq \frac{2C \lvert\mathcal{C}(z_i) \rvert^{1-\alpha}}{T(z_i)^{1-\alpha}} \\
 \mathbb{P}_{\mathbf{U'}} (A^c) &\leq \frac{2(K-1)C \lvert\mathcal{C}(z_i) \rvert^{1-\alpha}}{T(z_i)^{1-\alpha}} \numberthis \label{eq:lb2eq}
\end{align*}
Now from Lemma~\ref{lem:tsybakov} we have, 
\begin{align*}
&\mathrm{KL}\left(\mathbb{P} ^T _{\mathbf{U}},\mathbb{P} ^T _{\mathbf{U'}} \right) \\
&\geq (1-\alpha)\left(\log(T(z_i)) - \log(L/m)\right) - \log(4KC)
\end{align*}
Using standard methods from the bandit literature it can be shown that, 
\begin{align*}
\mathrm{KL}\left(\mathbb{P} ^T _{\mathbf{U}},\mathbb{P} ^T _{\mathbf{U'}} \right) = \sum _{s \in \mathcal{C}(z_i)} \sum_{\{t:S_t = s\}} \mathrm{KL}\left(U_{sk},\lambda \right)\mathbb{E}_{\mathbf{U}} \left[ \mathds{1} \left\{ X_t = k \right\}\right]
\end{align*}
Let us define the regret incurred during the time-steps where $S_t \in \mathcal{C}(z_i)$ as $R(T(z_i))$. We can follow the same procedure for all the sub-optimal arms which yields the following bound,
\begin{align*}
&R(T(z_i)) \geq \Delta(z_i) \sum_{k \neq k^*(z_i)}\sum _{s \in \mathcal{C}(z_i)} \sum_{\{t:S_t = s\}} \mathbb{E}_{\mathbf{U}} \left[ \left\{ X_t = k \right\}\right]  \\
 & \geq \left(\argmin_{k} \frac{(K-1)\Delta(z_i)}{  \mathrm{KL}\left(U_{sk},\lambda \right)}\right)\left((1-\alpha)\left(\log(T(z_i)) - \log(L/m)\right) \right. \left. - \log(4KC)\right) 
\end{align*}
Let $D(\mathbf{U}) =  \left(\argmin_{z_i,k} \frac{(K-1)\Delta(z_i)}{  \mathrm{KL}\left(U_{sk},\lambda \right)}\right)$. Now, we have
\begin{align*}
&R(T) = \sum_{z \in \mathcal{Z}}\mathbb{E} \left[R(T(z_i))\right] \\
& \geq D(\mathbf{U})(K-1)\mathbb{E} \sum_{z \in \mathcal{Z}}\left((1-\alpha)\left(\log(T(z_i)) - \log(L/m)\right) \right. \left.- \log(4KC)\right) 
\end{align*}
Now, using the fact that $ T(z_i) \geq \frac{T}{2m}$ given $E$, we have
\begin{align*}
&R(T) = \sum_{z \in \mathcal{Z}}\mathbb{E} \left[R(T(z_i))\right] \\
&= \sum_{z \in \mathcal{Z}}\mathbb{E} \left[R(T(z_i)) \lvert E\right]\mathbb{P}(E) + \mathbb{E} \left[R(T(z_i)) \lvert E ^c\right]\mathbb{P}(E^c) \\
&\geq D(\mathbf{U})(K-1)m \left((1-\alpha)\left(\log(T/2m) - \log(L/m)\right) - \right.  \left. \log(4KC)\right) + o(1)
\end{align*}
\end{proof}

\end{document}